\documentclass{article}

\usepackage[final]{neurips_2021}

\usepackage{mymath}

\usepackage[utf8]{inputenc} 
\usepackage[T1]{fontenc}    
\usepackage{url}            
\usepackage{booktabs}       
\usepackage{amsfonts}       
\usepackage{nicefrac}       
\usepackage{microtype}      

\title{Accommodating Picky Customers: Regret Bound and Exploration Complexity for Multi-Objective Reinforcement Learning}

\author{%
  Jingfeng Wu \\
  Department of Computer Science\\
  Johns Hopkins University\\
  Baltimore, MD 21218 \\
  \texttt{uuujf@jhu.edu} \\
  \And
  Vladimir Braverman \\
  Department of Computer Science \\
  Johns Hopkins University\\
  Baltimore, MD 21218 \\
  \texttt{vova@cs.jhu.edu} \\
  \AND
  Lin F. Yang \\
  Department of Electrical and Computer Engineering \\
  University of California, Los Angeles\\
  Los Angeles, CA 90095 \\
  \texttt{linyang@ee.ucla.edu}
}

\begin{document}

\maketitle

\begin{abstract}
In this paper we consider multi-objective reinforcement learning where the objectives are balanced using preferences. In practice, the preferences are often given in an adversarial manner, e.g., customers can be picky in many applications. We formalize this problem as an episodic learning problem on a Markov decision process, where transitions are unknown and a reward function is the inner product of a preference vector with pre-specified multi-objective reward functions. We consider two settings. In the online setting, the agent receives a (adversarial) preference every episode and proposes policies to interact with the environment. We provide a model-based algorithm that achieves a nearly minimax optimal regret bound $\widetilde{\mathcal{O}}\bigl(\sqrt{\min\{d,S\}\cdot H^2 SAK}\bigr)$, where $d$ is the number of objectives, $S$ is the number of states, $A$ is the number of actions, $H$ is the length of the horizon, and $K$ is the number of episodes. Furthermore, we consider preference-free exploration, i.e., the agent first interacts with the environment without specifying any preference and then is able to accommodate arbitrary preference vector up to $\epsilon$ error. Our proposed algorithm is provably efficient with a nearly optimal trajectory complexity $\widetilde{\mathcal{O}}\bigl({\min\{d,S\}\cdot H^3 SA}/{\epsilon^2}\bigr)$. This result partly resolves an open problem raised by \citet{jin2020reward}.
\end{abstract}

\section{Introduction}
In single-objective reinforcement learning (RL), a \emph{scalar reward} is pre-specified and an agent learns a policy to maximize the long-term cumulative reward~\citep{sutton2018reinforcement}.
However, in many real-world applications, we  need to optimize multiple objectives for the same (unknown) environment, even when these objectives are possibly contradicting~\citep{roijers2013survey}. 
For example, in an autonomous driving application, each passenger may have a different preference of driving styles: some of the passengers prefer a very steady riding experience while other passengers enjoy the fast acceleration of the car.
Therefore, traditional single-objective RL approach may fail to be applied in such scenarios.
One way to tackle this issue is the \emph{multi-objective reinforcement learning} (MORL)~\citep{roijers2013survey,yang2019generalized,natarajan2005dynamic,abels2018dynamic} method,
which models the multiple objectives by a \emph{vectorized reward}, and an additional \emph{preference vector} to specify the relative importance of each objective.
The agent of MORL needs to find policies to optimize the cumulative preference-weighted rewards.

If the preference vector is fixed or drawn from a fixed distribution, MORL is no more challenging than single-objective RL, as we can predict the objective to be optimized and apply (variants of) single-objective RL algorithms (e.g., \citep{azar2017minimax}).
However, more often in practice, the preference vector (under which the weighted objective needs to be optimized) is:
\begin{enumerate}[label=(\roman*),noitemsep,topsep=-1mm,parsep=0mm,partopsep=-1mm,leftmargin=*]
    \item adversarially provided,
    \item or even not available in the learning phase.
\end{enumerate}
Once more taking the autonomous driving application as example: (i) the intelligent system needs to adapt driving configurations to accommodate every customer, even though the next customer can have picky preference that is unpredictable; (ii) when the intelligent system is under development, the future customer cannot be known in advance, nonetheless, when the system is deployed, it has to be capable to accommodate any potential customer who rides the car.
Such MORL examples are common, to name a few beyond the autonomous driving one: medical treatment must take care of every patient even in very rare health conditions; an education system should accommodate every student according to his/her own characteristics; and emergency response systems have to be responsible in all extreme cases.
Due to these exclusive challenges that have not appeared in single-objective RL, new \emph{sample-efficient} algorithms for MORL need to be developed, as well as their \emph{theoretical grounds} need to be established.

In this work, we study \emph{provable} sample-efficient algorithms for MORL that 
resolve the aforementioned 
issues. 
Specifically, we consider MORL on a finite-horizon Markov decision process (MDP) with an unknown transition kernel, $S$ states, $A$ actions, $H$ steps, and $d$ reward functions that represent the $d$ difference objectives.
We investigate MORL problems in two paradigms: (i) \emph{online MORL} (where the preferences are adversarially presented), and (ii) \emph{preference-free exploration} (where the preferences are not available in the learning/exploration phase).
The two settings and our contributions are explained respectively in the following.

\textbf{Setting (i): Online MORL.}
We first consider an online learning problem to capture the challenge that preference vectors could be adversarially provided in MORL.
In the beginning of each episode, the MORL agent is provided (potentially adversarially) with a preference vector, and the agent interacts with the unknown environment to collect data and rewards (that is specified by the provided preference).
The performance of an algorithm is measured by the \emph{regret}, i.e., the difference between the rewards collected by the agent and those would be collected by a theoretically optimal agent (who could use varying policies that adapt to the preferences).
This setting generalizes the classical online single-objective RL problem~\citep{azar2017minimax} (where $d=1$).

\textbf{Contribution (i).}
For online MORL, we provide a provably efficient algorithm with a regret upper bound $\widetilde{\Ocal}\bigl(\sqrt{\min\{d,S\}\cdot H^2 SAK}\bigr)$\footnote{We use $\widetilde{\Ocal}(\cdot)$ to hide (potential) polylogarithmic factors in $\Ocal(\cdot)$, i.e., $\widetilde{\Ocal}(n) := \Ocal (n \log^k n)$ for sufficiently large $n$ and some absolute constant $k>0$.}, where $K$ is the number of episodes for interacting with the environment.
Furthermore, we show an information-theoretic lower bound $\Omega\bigl(\sqrt{\min\{d,S\}\cdot H^2 SAK}\bigr)$ for online MORL.
These bounds together show that, ignoring logarithmic factors, our algorithm resolves the online MORL problems optimally.

\textbf{Setting (ii): Preference-Free Exploration.}
We further consider an unsupervised MORL problem to capture the issue that preferences could be hard to obtain in the training phase.
The MORL agent first interacts with the unknown environment in the absence of preference vectors;
afterwards, the agent is required to use the collected data to compute near-optimal policies for an arbitrarily specified preference vector.
The performance is measured by the \emph{sample complexity}, i.e., the minimum amount of trajectories that an MORL agent needs to collect during exploration in order to be near-optimal during planning.
This setting extends the recent proposed \emph{reward-free exploration} problem~\citep{jin2020reward} to MORL.

\textbf{Contribution (ii).}
For preference-free exploration, we show that a simple variant of the proposed online algorithm can achieve nearly optimal sample complexity.
In particular, the algorithm achieves a sample complexity upper bound $\widetilde{\Ocal}\bigl({\min\{d,S\}\cdot H^3 SA}/{\epsilon^2}\bigr)$ where $\epsilon$ is the tolerance of the planning error; and we also show a sample  complexity lower bound, $\Omega\bigl({\min\{d,S\}\cdot H^2 SA}/{\epsilon^2}\bigr)$, for any algorithm.
These bounds suggest that our algorithm is optimal in terms of $d$, $S$, $A$, $\epsilon$ up to logarithmic factors.
It is also worth noting that our results for preference-free exploration partly answer an open question raised by \citet{jin2020reward}: reward-free RL is easier when the unknown reward functions enjoy low-dimensional representations (as in MORL).

\paragraph{Paper Layout.}
The remaining paper is organized as follows:
the preliminaries are summarized in Section~\ref{sec:preliminary};
then in Section~\ref{sec:online}, we formally introduce the problem of online MORL, our algorithm and its regret upper and lower bounds; 
then in Section~\ref{sec:explore}, we turn to study the preference-free exploration problem in MORL, where we present an exploration algorithm with sample complexity analysis (with both upper bound and lower bound), and compare our results with existing results for related problems;
finally, the related literature is reviewed in Section~\ref{sec:related}
and the paper is concluded by Section~\ref{sec:conclusion}.

\section{Preliminaries}\label{sec:preliminary}
We specify a  finite-horizon Markov decision process (MDP) by a tuple of $\bracket{\Scal, \Acal, H, \Pbb, \rB, \Wcal}$.
$\Scal$ is a finite \emph{state} set where $\abs{\Scal} = S$.
$\Acal$ is a finite \emph{action} set where $\abs{\Acal} = A$.
$H$ is the length of the horizon.
$\Pbb(\cdot \mid x,a)$ is a \emph{stationary, unknown transition probability} to a new state for taking action $a$ at state $x$.
$\rB  = \set{\rB_1, \dots, \rB_H}$, where
$\rB_h:\Scal\times \Acal \to [0,1]^d$ represents a $d$-dimensional \emph{vector rewards function} that captures the $d$ objectives\footnote{For the sake of presentation, we discuss bounded deterministic reward functions in this work. The techniques can be readily extended to stochastic reward settings.}. 
$\Wcal := \set{w\in [0,1]^d,\ \norm{w}_1 = 1}$ specifies the set of all possible \emph{preference vectors}\footnote{The condition that $w\in[0,1]^d$ is only assumed for convenience. Our results naturally generalize to preference vectors that are entry-wisely bounded by absolute constants.},
where each preference vector $w\in\Wcal$ induces a scalar reward function by\footnote{The linear scalarization method can be generalized. See more discussions in Section \ref{sec:regret-bound}, Remark \ref{rmk:scalarization}.}
\(
r_h (\cdot, \cdot) = \abracket{w, \rB_h(\cdot, \cdot)}
\)
for $h=1,\dots, H$.
A \emph{policy} is represented by $\pi := \set{\pi_1,\dots, \pi_H }$, where each $\pi_h(\cdot)$ maps a state to a distribution over the action set.
Fixing a policy $\pi$, we will consider the following generalized \emph{$Q$-value function} and generalized \emph{value function}:
\[
    Q^\pi_h (x,a; w) := \Ebb\bigl[\textstyle{\sum}_{j=h}^H \abracket{w, \rB_j(x_j, a_j)} \big| x_h = x, a_h=a\bigr],\quad 
    V^\pi_h (x; w) := Q^\pi_h (x, \pi_h(x); w),
\]
where $x_j \sim \prob{\cdot \mid x_{j-1}, a_{j-1}}$ and $a_j \sim \pi_j(x_j)$ for $j > h$.
Note that compared with the typical $Q$-value function (or the value function) used in single-objective RL, here the generalized $Q$-value function (or the generalized value function) takes the preference vector as an additional input, besides the state-action pair.
Fixing a preference $w\in\Wcal$, the optimal policy under $w$ is defined as
\(
    \pi^*_w := \arg\max_{\pi} V^\pi_1(x_1; w).
\)
For simplicity, we denote 
\begin{equation*}
    Q^*_h(x, a; w) := Q^{\pi^*_w}_h (x,a ; w), \quad
    V^*_h(x; w) := V^{\pi^*_w}_h(x ; w) = \max_a Q^{*}_h (x, a ; w).
\end{equation*}
The following abbreviation is adopted for simplicity:
\[\eval{\Pbb V^\pi_h}_{x,a,w} := \textstyle{\sum}_{y\in\Scal} \Pbb(y| x,a) V^\pi_h(y;w), \]
and similar abbreviations will be adopted for variants of probability transition (e.g., the empirical transition probability) and variants of value function (e.g., the estimated value function).
Finally, we remark that the following Bellman equations hold for the generalized $Q$-value functions
\begin{equation*}
    \begin{cases}
    Q^\pi_h(x,a;w) = \abracket{w, \rB_h(x,a)} + \eval{\Pbb V^\pi_{h+1}}_{x,a,w}, \\
        V^{\pi}_h (x; w) = Q^{\pi}_h(x, \pi_h(x); w);   
    \end{cases}
    \begin{cases}
    Q^*_h(x,a;w) = \abracket{w, \rB_h(x,a)} + \eval{\Pbb V^*_{h+1}}_{x,a,w}, \\
        V^*_h(x;w) = \max_{a} Q^*_h(x,a;w).
    \end{cases}
\end{equation*}
With the above preparations, we are ready to discuss our algorithms and theory for online MORL (Section \ref{sec:online}) and preference-free exploration (Section \ref{sec:explore}).

\section{Online MORL}\label{sec:online}

\paragraph{Problem Setups.}
The online setting captures the first difficulty in MORL, where the preference vectors can be adversarially provided to the MORL agent.
Formally, the MORL agent interacts with an unknown environment through Protocol~\ref{alg:mo-protocol}:
at the beginning of the $k$-th episode, an adversary selects a preference vector $w^k$ and reveals it to the agent;
then starting from a fixed initial state\footnote{Without loss of generality, we fix the initial state; otherwise we may as well consider an MDP with an external initial state $x_0$ with zero reward for all actions, and a transition $\Pbb_0 (\cdot \mid x_0, a) = \Pbb_0(\cdot) $ for all action $a$. This is equivalent to our setting by letting the horizon length $H$ be $H+1$.} $x_1$,
the agent draws a trajectory from the environment by recursively taking an action and observing a new state, and collects rewards from the trajectory, where the rewards are scalarized from the vector rewards by the given preference.
The agent's goal is to find the policies that maximize the cumulative rewards.
Its performance will be measured by the following \emph{regret}:
suppose the MORL agent has interacted with the environment through Protocol \ref{alg:mo-protocol} for $K$ episodes, where the provided preferences are $\{w^k\}_{k=1}^K$ and the adopted policies are $\{\pi^k\}_{k=1}^K$ correspondingly,
we consider the regret of the collected rewards (in expectation) competing with the theoretically maximum collected rewards (in expectation):
\begin{equation}\label{eq:regret}
    \regret (K) := \textstyle{\sum}_{k=1}^K V^*_1 (x_1 ; w^k) - V^{\pi^k}_1 (x_1 ; w^k).
\end{equation}
Clearly, the regret \eqref{eq:regret} is always non-negative, and a smaller regret implies a better online performance.
We would like to highlight that the regret \eqref{eq:regret} allows the theoretically optimal agent to adopt varying policies that adapt to the preferences.

\begin{protocol}[thb!]
    \caption{\texttt{Environment Interaction Protocol}}
    \label{alg:mo-protocol}
    \begin{algorithmic}[1]
        \REQUIRE MDP$\bracket{\Scal, \Acal, H, \Pbb, \rB, \Wcal}$ and online agent
        \STATE agent observes $\bracket{\Scal, \Acal, H, \rB, \Wcal}$
        \FOR{episode $k = 1, 2, \dots , K$}
            \STATE agent receives an initial state $x^k_1 = x_1$, and a preference $w^k$ (from an adversary)
            \FOR{step $h = 1, 2, \dots , H$}
                \STATE agent chooses an action $a^k_h$, and collects reward $\abracket{w^k, \rB_h(x^k_h, a^k_h)}$
                \STATE agent transits to a new state $x^k_{h+1}\sim \Pbb(\cdot \mid x^k_h, a^k_h)$
            \ENDFOR
        \ENDFOR
    \end{algorithmic}
\end{protocol}

\begin{wrapfigure}{r}{0.45\textwidth}
\vspace{-0.5cm}
    \centering
    \includegraphics[width=1.0\linewidth]{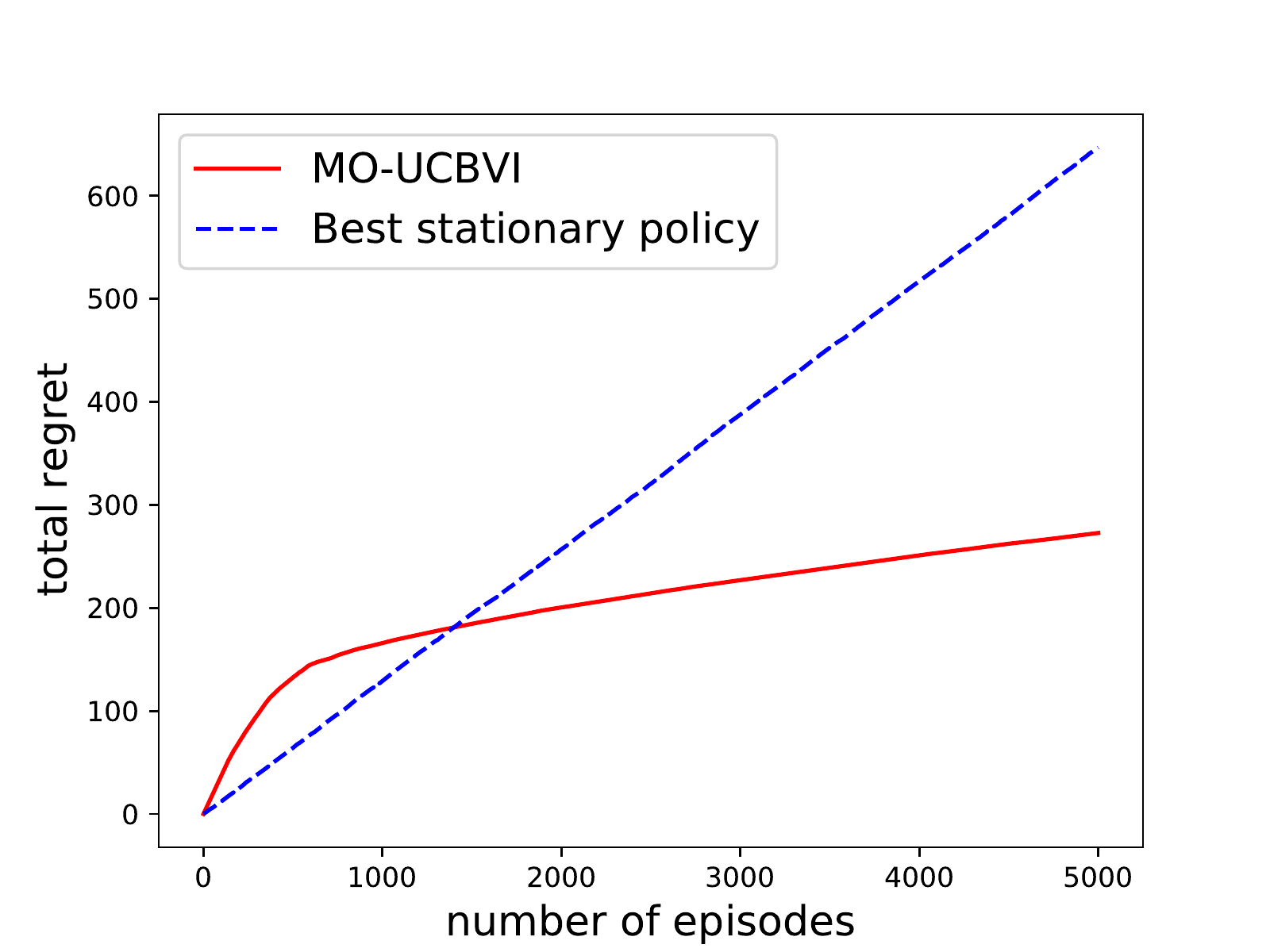}
    \caption{\small A regret comparison of \MOVI vs. the best-in-hindsight policy in a simulated random multi-objective MDP. Note that the best-in-hindsight policy is the optimal policy for single-objective RL. The plots show that the best-in-hindsight policy will incur linear regret in online MORL, and the proposed \MOVI achieves sublinear regret as predicted by Theorem~\ref{thm:regret-informal}. See Appendix~\ref{append-sec:experiment} for details.}
    \label{fig:simulation}
    \vspace{-0.5cm}
\end{wrapfigure}

\paragraph{Connections to Online Single-Objective RL.}
The online regret minimization problems are well investigated in the context of single-objective RL (see, e.g., \citep{azar2017minimax,jin2018q}).
In both online single-objective RL and our studied online MORL, it is typical to assume that the transition probability $\Pbb$ is the only unknown information about the environment since estimating a stochastic reward is relatively easy (see, e.g., \citep{azar2017minimax,jin2018q}).
In particular, single-objective RL is a special case of MORL in the online setting, where the preference is fixed during the entire learning process (i.e., $w^k := w$ for all $k$) --- therefore an online MORL algorithm naturally applies to single-objective RL.
However, the reverse is not true as that in MORL the preference vectors can change over time and are potentially adversarial.

\paragraph{Comparison with Single-Objective Stochastic and Adversarial Reward Setting.} 
The essential difficulty of MORL is further reflected in the regret \eqref{eq:regret}.
Specifically, the regret \eqref{eq:regret} compares the performance of the agent's policies to a sequence of \emph{optimal} policy under \emph{each} given preference, which could \emph{vary} over time.
However, in online single-objective RL, either with known/stochastic rewards~\citep{azar2017minimax,jin2018q} or adversarial rewards~\citep{rosenberg2019online,jin2019learning}, the benchmark policy is supposed to be \emph{fixed} over time (the best policy in the hindsight).
This difference suggests that online MORL could be more challenging than online single-objective RL, as a harder performance measurement is adopted.
More specifically, when measured by regret \eqref{eq:regret}, existing algorithms for single-objective RL (e.g., \citep{azar2017minimax}) easily suffer a $\propto\Theta(K)$ regret when the sequence of preferences is adversarially designed; in contrast, we will show an MORL algorithm that experiences at most $\propto\widetilde\Ocal(\sqrt{K})$ regret under any sequence of preferences.
A numerical simulation of this issue is presented in Figure \ref{fig:simulation}.

\subsection{A Sample-Efficient Online Algorithm}

In this part we introduce an algorithm for online MORL, called \emph{multi-objective upper confidence bound value iteration} (\MOVI). A simplified verison of \MOVI is presented as Algorithm~\ref{alg:online}, and a more advanced version (Algorithm \ref{alg:online-Bernstein}) can be found in Appendix \ref{append-sec:proof-regret}.
Our algorithm is inspired by \texttt{UCBVI}~\citep{azar2017minimax} that achieves minimax regret bound in online single-objective RL.

\begin{algorithm}[thb!]
    \caption{\MOVI}
    \label{alg:online}
    \begin{algorithmic}[1]
    \STATE initialize history $\Hcal^0 = \emptyset$
    \FOR{episode $k=1,2,\dots,K$}
        \STATE $N^k(x,a),\ \widehat{\Pbb}^k (y \mid x,a)\leftarrow \EmpiTrans(\Hcal^{k-1})$\label{line:trans}        
        \STATE compute bonus \( b^k(x,a) := c\cdot \sqrt{\frac{\min\set{d,S} H^2 \iota }{N^k(x,a)} },\) where $\iota =  \log (HSAK/\delta)$ and $c$ is a constant \label{line:def-bonus}
        \STATE receive a preference $w^k$        
        \STATE $\{Q^k_h(x,a;w^k)\}_{h=1}^H \leftarrow \UCBQ(\widehat{\Pbb}^{k}, w^k, b^k)$\label{line:ucb}
        \STATE receive initial state $x_1^k = x_1$
        \FOR{step $h=1,2,\dots,H$}
            \STATE take action $a^k_h = \arg\max_a Q^k_h(x^k_h, a ; w^k)$, and obtain a new state $x^{k}_{h+1}$ \label{line:run-policy}
        \ENDFOR 
        \STATE update history $\Hcal^k = \Hcal^{k-1} \cup \{x^k_h, a^k_h\}_{h=1}^H$\label{line:history} 
    \ENDFOR

    \item[]
    
    \STATE \textbf{Function} \EmpiTrans \label{line:empi-trans}
    \STATE \textbf{Require:} history $\Hcal^{k-1}$ 
    \FOR{$(x,a,y) \in \Scal \times \Acal \times \Scal$}
         \STATE \(N^k (x,a,y) := \# \{(x,a,y) \in \Hcal^{k-1} \} \), and \( N^k(x,a) := \sum_{y} N^k (x,a,y) \)
        \IF{$N^k(x,a) > 0$}
        \STATE $\widehat{\Pbb}^k(y\mid x,a) = {N^k(x,a,y)} / {N^k(x,a)}$
        \ELSE 
        \STATE $\widehat{\Pbb}^k(y\mid x,a) = 1/S$
        \ENDIF
    \ENDFOR
    \RETURN $N^k(x,a)$ and $\widehat{\Pbb}^k(y\mid x,a)$

    \item[]

    \STATE \textbf{Function} \UCBQ \label{line:ucbq}
    \STATE \textbf{Require:} empirical transition $\widehat\Pbb^k$, preference $w^k$, and bonus $b^k(x,a)$
        \STATE set $V^k_{H+1}(x ; w^k) = 0$ 
        \FOR{step $h = H, H-1, \dots , 1$}
            \FOR{$(x,a) \in \Scal \times \Acal$}
                \STATE $Q^k_h(x,a ; w^k) = \min\big\{H, \ \abracket{w^k, \rB_h(x, a) } + b^k(x,a) + {\widehat{\Pbb}^k_h V^k_{h+1}}\big|_{x,a, w^k} \big\}$\label{line:add-bonus}
                \STATE $V^k_h(x ; w^k)=\max_{a\in \Acal}Q^k_h(x,a; w^k)$
            \ENDFOR
        \ENDFOR
        \RETURN $\big\{Q^k_h(x,a; w^k)\big\}_{h=1}^H$
    \end{algorithmic}
\end{algorithm}

In Algorithm~\ref{alg:online}, the agent interacts with the environment according to Protocol~\ref{alg:mo-protocol}, and use an \emph{optimistic policy} to explore the unknown environment and collect rewards.
The optimistic policy is a greedy policy with respect to an optimistic estimation to the value function, {\UCBQ} (lines~\ref{line:ucb} and \ref{line:run-policy}).
Specifically, {\UCBQ} is constructed to maximize the cumulative reward, which is scalarized by the current preference, through dynamic programming over an empirical transition probability (line \ref{line:ucbq}).
The empirical transition probability is inferred from the data collected so far (lines~\ref{line:trans} and \ref{line:empi-trans}), which might not be accurate if a state-action pair has not yet been visited for sufficient times.
To mitigate this inaccuracy, \UCBQ utilizes an extra \emph{exploration bonus} (lines~\ref{line:def-bonus} and \ref{line:add-bonus}) so that:
(i) \UCBQ never under-estimates the optimal true value function, for \emph{whatever preference vector} (and with high probability);
and (ii) the added bonus \emph{shrinks quickly} enough (as the corresponding state-action pair continuously being visited) so that the over-estimation is under control.
The overall consequence is that: \MOVI explores the environment via a sequence of optimistic policies, in order to collect rewards under a sequence of adversarially provided preferences; since the policies are optimistic for any preference, the incurred regret would not exceed the total amount of added bonus; and since the bonus decays fast, their sum up would be sublinear (with respect to the number of episodes played). 
Therefore \MOVI only suffers a sublinear regret, even when the preferences are adversarially presented. The above intuition is formalized in the next part.


\subsection{Theoretical Analysis}\label{sec:regret-bound}
The next two theorems justify regret upper bounds for \MOVI and a regret lower bound for online MORL problems, respectively.
\begin{thm}[Regret bounds for \MOVI]\label{thm:regret-informal}
    Suppose $K$ is sufficiently large. Then for any sequence of the incoming preferences $\{w^k\}_{k=1}^K$, with probability at least $1-\delta$:
    \begin{itemize}[noitemsep,topsep=-1mm,parsep=0mm,partopsep=-1mm,leftmargin=*]
        \item the regret~\eqref{eq:regret} of \MOVI (see Algorithm~\ref{alg:online}) satisfies
    \[
     \regret (K) \le {\Ocal} \bigl(\sqrt{  \min\{d,S\} \cdot H^3 SA K \log ( {HSAK}/{\delta}}  )\bigr);
    \]
    \item and the regret~\eqref{eq:regret} of a Bernstein-variant \MOVI (see Algorithm~\ref{alg:online-Bernstein} in Appendix \ref{append-sec:proof-regret}) satisfies
    \[
     \regret (K) \le {\Ocal} \bigl(\sqrt{  \min\{d,S\} \cdot H^2 SA K \log^2 ( {HSAK}/{\delta}}  )\bigr).
    \]
    \end{itemize}
\end{thm}

\begin{thm}[A regret lower bound for MORL]\label{thm:regret-lower-bound-informal}
There exist some absolute constants $c, K_0 > 0$, such that for any $K > K_0$, any MORL algorithm that runs $K$ episodes, there is a set of MOMDPs and a sequence of (necessarily adversarially chosen) preferences vectors such that 
\[ 
     \Ebb[ \regret(K) ] \ge c\cdot\sqrt{ \min\set{d,S}\cdot H^2 SA K }, 
\]
where the expectation is taken with respect to the randomness of choosing MOMDPs and the randomness of the algorithm for collecting dataset.
\end{thm}

\begin{rmk}
When $d=1$, MORL recovers the single-objective RL setting, and Theorem \ref{thm:regret-informal} recovers existing nearly minimax regret bounds for single-objective RL \citep{azar2017minimax,zanette2019tighter}.
Moreover, the lower bound in Theorem \ref{thm:regret-lower-bound-informal} implies that our upper bound in Theorem~\ref{thm:regret-informal} is tight ignoring logarithm terms.
Interestingly, the lower bound suggests MORL with $d>2$ is truly harder than single objective RL (corresponding to $d=1$) as the sequence of preferences can be adversarially chosen.
\end{rmk}

\begin{rmk}\label{rmk:scalarization}
Theorem~\ref{thm:regret-informal} (as well as our other theorems) applies to general scalarization methods besides the linear one as adopted by Algorithm~\ref{alg:online} and many other MORL papers~\citep{yang2019generalized,roijers2013survey,natarajan2005dynamic,abels2018dynamic}.
In particular, our results apply to scalarization functions $r_h(\cdot, \cdot) = f(\rB_h(\cdot,\cdot); w)$ that are (1) deterministic, (2) Lipschitz continuous for $w$, and (3) bounded between $[0,1]$ (which can be relaxed). 
This will be clear from proofs in Appendix \ref{append-sec:proof-regret}, where we treat the potentially adversarially given preferences by a covering argument and union bound, and these techniques are not dedicated to linear scalarization function and can be easily extended to more general cases. 
\end{rmk}


The proof of Theorem \ref{thm:regret-informal} leverages standard analysis procedures for single-objective RL~\citep{azar2017minimax,zanette2019tighter}, and a covering argument with an union bound to tackle the challenge of adversarial preferences.
Th rigorous proof is included in Appendix \ref{append-sec:proof-regret}.

Specifying an adversarial process of providing preferences is the key challenge for proving Theorem~\ref{thm:regret-lower-bound-informal}.
To handle this issue, we use reduction techniques and utilize a lower bound that we will show shortly for preference-free exploration problems. We refer the readers to Theorem \ref{thm:exploration-lower-bound-informal} and Appendix \ref{append-sec:regret-lower-bound} for more details.

\section{Preference-Free Exploration}\label{sec:explore}

\paragraph{Problem Setups.}
\emph{Preference-free exploration} (PFE) captures the second difficulty in MORL: the preference vector might not be observable when the agent explores the environment.
Specifically, PFE consists of an exploration phase and a planning phase.
Similarly as in the online setting, the transition probability is hidden from the MORL agent.
In the exploration phase, the agent interacts with the unknown environment to collect samples, however the agent has no information about the preference vectors at this point.
Afterwards PFE switches to the planning phase, where the agent is prohibited to obtain new data, and is required to compute near-optimal policy for any preference-weighted reward functions.
Since this task is no longer in an online fashion, we turn to measure the performance of a PFE algorithm by the minimum number of required trajectories (in the exploration phase) so that the algorithm can behave near-optimally in the planning phase.
This is made formal as follows:
a PFE algorithm is called $(\epsilon,\delta)$-PAC (Probably Approximately Correct),
if 
\begin{equation*}
    \Pbb\bigl\{ \forall w\in\Wcal,\ V^*_1(x_1; w) - V^{\pi_w}_1(x_1;w) \le \epsilon  \bigr  \}
    \ge 1-\delta,
\end{equation*}
where $\pi_w$ is the policy outputted by the PFE algorithm for input preference $w$, then the \emph{sample complexity} of a PFE algorithm is defined by the least amount of trajectories it needs to collect in the exploration phase for being $(\epsilon,\delta)$-PAC in the planning phase.

\paragraph{Connections to Reward-Free Exploration.}
PFE problem is a natural extension to the recent proposed \emph{reward-free exploration} (RFE) problem \citep{jin2020reward,kaufmann2020adaptive,wang2020reward,menard2020fast,zhang2020nearly}. 
Both problems consist of an exploration phase and a planning phase; the difference is in the planning phase: in RFE, the agent needs to be able to compute near-optimal policies for \emph{all reward functions}, while in PFE, the agent only needs to achieve that for \emph{all preference-weighted reward functions}, i.e., the reward functions that can be represented as the inner product of a $d$-dimensional preference vectors and the $d$-dimensional vector rewards functions (i.e., the $d$ objectives in MORL).
A PFE problem reduces to a RFE problem if $d=SA$ such that every reward function can be represented as a preference-weighted reward function.
However, if $d \ll SA$, it is conjectured by \citet{jin2020reward} that PFE can be solved with a much smaller sample complexity than RFE. 
Indeed, in the following part we show an algorithm that improves a $\propto\widetilde\Ocal(S^2)$ dependence for RFE to $\propto \widetilde\Ocal(\min\{d,S\}\cdot S)$ dependence for PFE, in terms of sample complexity.


\subsection{A Sample-Efficient Exploration Algorithm}
We now present a simple variant of \MOVI that is sample-efficient for PFE.
The algorithm is called \emph{preference-free upper confidence bound exploration} (\PFUCB), and is discussed separately as in the exploration phase (Algorithm \ref{alg:exploration}) and in the planning phase (Algorithm \ref{alg:planning}) in below. 

\begin{algorithm}
    \caption{\PFUCB (Exploration)}
    \label{alg:exploration}
    \begin{algorithmic}[1]
        \STATE initialize history $\Hcal^0 = \emptyset$
        \FOR{episode $k=1,2,\dots,K$}
            \STATE $N^k(x,a),\ \widehat{\Pbb}^k (y \mid x,a)\leftarrow \EmpiTrans(\Hcal^{k-1})$
            \STATE compute bonus
            \( c^k(x,a) := \frac{H^2S}{2N^k(x,a)} + 2 b^k(x,a) \) for $b^k(x,a)$ defined in Algorithms \ref{alg:online} or \ref{alg:planning}\label{line:explore-bonus}
            \STATE $\{\overline{Q}^k_h(x,a)\}_{h=1}^H \leftarrow \UCBQ(\widehat{\Pbb}^{k}, w=0, c^k)$\label{line:explore-ucb}
            \STATE receive initial state $x_1^k = x_1$
            \FOR{step $h=1,2,\dots,H$}
                \STATE take action $a^k_h = \arg\max_a \overline{Q}^k_h(x^k_h, a)$, and obtain a new state $x^{k}_{h+1}$
            \ENDFOR 
            \STATE update history $\Hcal^k = \Hcal^{k-1} \cup \{x^k_h, a^k_h\}_{h=1}^H$
        \ENDFOR
    \end{algorithmic}
\end{algorithm}


\begin{algorithm}
    \caption{\PFUCB (Planning)}
    \label{alg:planning}
    \begin{algorithmic}[1]
    \REQUIRE history $\Hcal^K$, preference vector $w$
    \FOR{$k=1,2,\dots,K$}
        \STATE $N^k(x,a),\ \widehat{\Pbb}^k (y | x,a)\leftarrow \EmpiTrans(\Hcal^{k-1})$
        \STATE compute bonus \( b^k(x,a) := c\cdot \sqrt{\frac{\min\set{d,S} H^2 \iota }{N^k(x,a)} }\) where $\iota =  \log (HSAK/\delta)$ and $c$ is a constant
        \STATE $\{Q^k_h(\cdot,\cdot;w)\}_{h=1}^H \leftarrow \UCBQ(\widehat{\Pbb}^{k}, w, b^k)$\label{line:planning-ucb}
        \STATE infer greedy policy $\pi^k_h (x) = \arg\max_a Q^k_h(x,a;w) $
    \ENDFOR
    \RETURN $\pi$ drawn uniformly from $\{\pi^k\}_{k=1}^K$
    \end{algorithmic}
\end{algorithm}

Algorithm~\ref{alg:exploration} presents our PFE algorithm in the exploration phase.
Indeed, Algorithm~\ref{alg:exploration} is a modified \MOVI (Algorithm~\ref{alg:online}) by setting the preference to be zero (Algorithm~\ref{alg:exploration}, line~\ref{line:explore-ucb}), and slightly enlarging the exploration bonus (Algorithm~\ref{alg:exploration}, line~\ref{line:explore-bonus}).
The intention of an increased  exploration bonus will be made clear later when we discuss the planning phase.
With a zero preference vector, the \UCBQ in Algorithm~\ref{alg:exploration} will identify a trajectory along which the cumulative bonus (instead of the cumulative rewards) is maximized (with respect to the empirical transition probability).
Also note that the bonus function (Algorithm~\ref{alg:exploration}, line~\ref{line:explore-bonus}) is negatively correlated with the number of visits to a state-action pair.
Hence the greedy policy with respect to the zero-preference \UCBQ tends to visit the state-actions pairs that are associated with large bonus, i.e., those that have been visited for less times.
In sum, Algorithm \ref{alg:exploration} explores the unknown environment ``uniformly'', without the guidance of preference vectors.

Then Algorithm~\ref{alg:planning} shows our PFE algorithm in the planning phase.
Given any preference vector, Algorithm~\ref{alg:planning} computes a sequence of optimistically estimated value functions based on the data collected from the exploration phase, and then outputs a greedy policy with respect to a randomly drawn optimistic value estimation.
Note that the bonus in Algorithm~\ref{alg:planning} is set as the one in Algorithm \ref{alg:online}, and recall that the bonus in Algorithm \ref{alg:exploration} is an enlarged one.
The relatively large bonus in the exploration phase guarantees that the regret in the planning phase never exceeds that in the exploration phase. 
On the other hand, based on Theorem \ref{thm:regret-informal} for \MOVI (Algorithm \ref{alg:online}), we have already known that the exploration algorithm (Algorithm \ref{alg:exploration}), a modified Algorithm \ref{alg:online}, suffers at most $\widetilde\Ocal(\sqrt{K})$ regret, hence the planning algorithm (Algorithm \ref{alg:planning}) experiences at most $\widetilde\Ocal(1/\sqrt{K})$ error.
The next part rigorously justifies these discussions.

\subsection{Theoretic Analysis}
We first provide Theorem \ref{thm:exploration-informal} to justify the trajectory complexity of Algorithms \ref{alg:exploration} and \ref{alg:planning} in the PFE setting; then we present Theorem \ref{thm:exploration-lower-bound-informal} that gives an information-theoretic lower bound on the trajectory complexity for any PFE algorithm.

\begin{thm}[A trajectory complexity of \PFUCB]\label{thm:exploration-informal}
Suppose $\epsilon>0$ is sufficiently small.
Then for \PFUCB (Algorithm~\ref{alg:exploration}) run for 
    \[ 
    K = \Ocal \bigl( {\min\{d,S\}\cdot H^3 SA \iota} /{\epsilon^2}  \bigr), \quad \text{where}\ \iota := \log ({HSA}/{(\delta\epsilon)}), \]
    episodes, \PFUCB (Algorithm~\ref{alg:planning}) is $(\epsilon, \delta)$-PAC.
\end{thm}

\begin{thm}[A lower bound for PFE]\label{thm:exploration-lower-bound-informal}
There exist absolute constants $c, \epsilon_0 > 0$, such that for any $0< \epsilon < \epsilon_0$, there exists a set of MOMDPs such that any PFE algorithm that is $(\epsilon, 0.1)$-PAC on them, it needs to collect at least
    \[ 
    K \ge \Omega \bigl( {\min\set{d,S}\cdot H^2 SA} / {\epsilon^2} \bigr)
    \]
    trajectories in expectation (with respect to the randomness of choosing MOMDPs and the exploration algorithm).
\end{thm}


\begin{rmk}
According to Theorems~\ref{thm:exploration-informal} and~\ref{thm:exploration-lower-bound-informal}, the trajectory complexity of \PFUCB is optimal for $d,S,A,\epsilon$ ignoring logarithmic factors, but is an $H$ factor loose compared with the lower bound. This is because the current Algorithm \ref{alg:exploration} utilizes a \emph{preference-independent}, Hoeffding-type bonus since the preference vector is not available during exploration. We leave it as an open problem to further remove this gap about $H$.
\end{rmk}


\textbf{Proof Sketch of Theorem~\ref{thm:exploration-informal}.}
Theorem~\ref{thm:exploration-informal} is obtained in three procedures.
(1) We first observe the total regret incurred by Algorithm \ref{alg:exploration} is $\widetilde{\Ocal}\big(\sqrt{\min\set{d,S} H^3 SA K}\big)$ according to Theorem \ref{thm:regret-informal}.
(2) Then utilizing the enlarged exploration bonus, we show that in each episode, the planning error is at most constant times of the incurred exploration error.
(3) Thus the averaged planning error is at most $\widetilde{\Ocal}\big(\sqrt{\min\set{d,S} H^3 SA / K}\big)$ as claimed.
A complete proof is deferred to Appendix \ref{append-sec:proof-explore}.

Note that the second argument is motivated by \citep{zhang2020task,wang2020reward}. 
However in their original paper, a brute-force union bound over all possible value functions are required to obtain similar effect, due to the limitation of model-free algorithm \citep{zhang2020task} (see Appendix~\ref{append-sec:discuss-zhang} for more details) or linear function approximation \citep{wang2020reward}. This will cause a loose, $\propto\widetilde\Ocal (S^2)$ complexity in the obtained bound.
Different from their approach, we carefully manipulate a lower order term to avoid union bounding all value functions during the second argument.
As a consequence we obtain a near-tight $\propto\widetilde\Ocal (\min\{d,S\}\cdot S)$ dependence in the final bound.
We believe this observation has broader application in the analysis of similar RL problems.

\textbf{Proof Sketch of Theorem~\ref{thm:exploration-lower-bound-informal}.}
We next introduce the idea of constructing the hard instance that witnesses the lower bound in Theorem \ref{thm:exploration-lower-bound-informal}.
A basic ingredient is the hard instance given by \citet{jin2020reward}.
However, this hard instance is invented for RFE, where the corresponding lower bound is $K \ge \Omega\bigl({ H^2 S^2 A}/{\epsilon^2}\bigr)$. Note this lower bound cannot match the upper bound in Theorem \ref{thm:exploration-informal} in terms of $d$ and $S$.
In order to develop a dedicated lower bound for PFE, we utilize Johnson–Lindenstrauss Lemma~\citep{johnson1984extensions} to refine the hard instance in \citep{jin2020reward}, and successfully reduce a factor $S$ to $\min\set{d, S}$ in their lower bound, which gives the result in Theorem \ref{thm:exploration-lower-bound-informal}.
We believe that the idea to refine RL hard instance by Johnson–Lindenstrauss Lemma is of broader interests.
A rigorous proof is deferred to Appendix \ref{append-sec:pfe-lower-bound}.

\paragraph{Application in Reward-Free Exploration.}
By setting $d=SA$ and allowing arbitrary reward functions, PFE problems reduce to RFE problems. Therefore as a side product, Theorem \ref{thm:exploration-informal} implies the following results for RFE problems on stationary or non-stationary MDPs\footnote{Our considered MDP is \emph{stationary} as the transition probability $\Pbb$ is fixed (across different steps). An MDP is called \emph{non-stationary}, if the transition probability varies at different steps, i.e., replacing $\Pbb$ by $\{\Pbb_h\}_{h=1}^H$.}:

\begin{cor}\label{thm:rfe}
Suppose $\epsilon>0$ is sufficiently small.
Consider the reward-free exploration problems on a \emph{stationary MDP}.
Suppose \PFUCB (Algorithm \ref{alg:exploration}) is run for 
    \[ 
    K = \Ocal \bigl( {H^3 S^2 A} \iota / {\epsilon^2} \bigr), \quad \text{where} \ \iota := \log ({HSA}/{(\delta\epsilon)} ),
    \]
    episodes, then \PFUCB (Algorithm~\ref{alg:planning}) is $(\epsilon, \delta)$-PAC.
    Moreover, if the MDP is \emph{non-stationary}, the above bound will be revised to
    \(
    K = \Ocal \bigl( {H^4 S^2 A \iota} / {\epsilon^2} \bigr).
    \)
\end{cor}

\begin{rmk}
When applied to RFE, \PFUCB matches the rate shown in \citep{kaufmann2020adaptive} and improves an $H$ factor compared with \citep{jin2020reward} (for both stationary and non-stationary MDPs), but is an $H$ factor loose compared with the current best rates, \citep{menard2020fast} (for non-stationary MDPs) and \citep{zhang2020nearly} (for stationary MDPs).
However we highlight that our results are superior in the context of PFE, since \PFUCB adapts with the structure of the rewards. 
In specific, if rewards admit a $d$-dimensional feature for $d < S$, \PFUCB only needs $\propto\widetilde\Ocal (\min\{d,S\}\cdot S)$ samples, 
but the above methods must explore the whole environment with a high precision which consumes $\propto\widetilde\Ocal (S^2)$ samples. 
\end{rmk}

\paragraph{Application in Task-Agnostic Exploration.}
PFE is also related to \emph{task-agnostic exploration} (TAE)~\citep{zhang2020task}: in PFE, the agent needs to plan for an arbitrary reward function from a fixed and bounded $d$-dimensional space; and in TAE, the agent needs to plan for $N$ fixed reward functions.
Due to the nature of the problem setups, PFE algorithms (that do not exploit the given reward basis $\rB$ during exploration, e.g., ours) and TAE algorithms can be easily applied to solve the other problem through a covering argument and a union bound, and with a modification of $\min\{d,S\} \leftrightarrow \log (N)$ in the obtained trajectory complexity bounds.
For TAE on a non-stationary MDP, \citet{zhang2020task} show an algorithm which takes $\widetilde\Ocal\bigl({\log(N) \cdot H^5 SA}/{\epsilon^2}\bigr)$ episodes for TAE\footnote{This bound is copied from \citep{zhang2020task}, which is erroneously stated due to a technical issue in the proof of Lemma 2 in their original paper. The issue can be fixed by a covering argument and union bound on the value functions, but then the obtained bound should be $\widetilde\Ocal\bigl({\log(N) H^5 S^2 A}/{\epsilon^2}\bigr)$. See Appendix~\ref{append-sec:discuss-zhang} for details.}.
In comparison, when applied to TAE on a non-stationary MDP, Theorem~\ref{thm:exploration-informal} implies \PFUCB only takes $\widetilde\Ocal\bigl({ \log (N) \cdot H^4 SA}/{\epsilon^2}\bigr)$ episodes\footnote{The conversion holds as follows. First set $d=1$ in our algorithm to yield an algorithm for TAE with a single agnostic task, where we have $\min\{d,S\}=1$. Then one can extend this algorithm to TAE with $N$ agnostic tasks using a union bound to have the algorithm succeed simultaneously for all $N$ tasks, which adds a $\log N$ multiplicative factor in the sample complexity bound. In this way, we obtain a TAE algorithm with a sample complexity bound in Theorem \ref{thm:exploration-informal} where $\min\{d,S\}$ is replaced with $\log N$.}, which improves~\citep{zhang2020task}.

\section{Related Works}\label{sec:related}

\textbf{MORL.}
MORL receives extensive attention from RL applications~\citep{yang2019generalized,natarajan2005dynamic,mossalam2016multi,abels2018dynamic,roijers2013survey}. 
However, little theoretical results are known, especially in terms of the sample efficiency and adversarial preference settings.
A large amount of MORL works focus on identifying (a cover for) the policies belonging to the \emph{Pareto front} \citep{yang2019generalized,roijers2013survey,cheung2019exploration}, but the correspondence between a preference and an optimal policy is ignored.
Hence they cannot ``accommodate picky customers'' as our algorithms do.
To our knowledge, this paper initiates the theoretical study of the sample efficiency for MORL in the setting of adversarial preferences.


\textbf{Adversarial MDP.}
Similarly to the online MORL problem studied in this paper, the adversarial MDP problem \citep{neu2012adversarial,rosenberg2019online,jin2019learning} also allows the reward function to change adversarially over time. 
However, we study a totally different regret (see \eqref{eq:regret}).
In the adversarial MDP problem, the regret is measured against a \emph{fixed} policy that is best in the hindsight;
but in online MORL problem, we study a regret against a sequence of optimal policies with respect to the sequence of incoming preferences, i.e., our benchmark policy may \emph{vary} over time.
Therefore our regret bound for online MORL problems is orthogonal to those for adversarial MDP problems~\citep{neu2012adversarial,rosenberg2019online,jin2019learning}.

\textbf{Constrained MDP.}
In the problem of the \emph{constrained MDP}, $d$ constraints are enforced to restrict the policy domain, and the agent aims to find a policy that belongs to the domain and maximizes the cumulative rewards \citep{achiam2017constrained,efroni2020exploration,el2018controlled,fisac2018general,wachi2020safe,garcia2015comprehensive,brantley2020constrained}.
Constrained MDP is related to MORL as when the constraints are soft, they can be formulated as ``objectives'' with negative weights.
However, there is a fundamental difference: constrained MDP aims to optimize only one (and known) objective, where in MORL studied in this paper, we aim to be able to find near optimal policies for any preference-weighted objective (to accommodate picky customers).

\textbf{Reward-Free Exploration.} 
The proposed preference-free exploration problem generalizes the problems of reward-free exploration~\citep{jin2020reward}.
Compared with existing works for RFE~\citep{jin2020reward,kaufmann2020adaptive,menard2020fast,zhang2020nearly}, our method has the advantage of adapting with rewards that admit low-dimensional structure --- this partly answers an open question raised by \citet{jin2020reward}.
We also note that \citet{wang2020reward} study RFE with linear function approximation on the value functions; in contrast our setting can be interpreted as RFE with linear function approximation on the reward functions, which is orthogonal to their setting.

\section{Conclusion}\label{sec:conclusion}
In this paper we study provably sample-efficient algorithms for multi-objective reinforcement learning in both online and preference-free exploration settings.
For both settings, sample-efficient algorithms are proposed and their sample complexity analysis is provided;
moreover, two information-theoretic lower bounds are proved to justify the near-tightness of the proposed algorithms, respectively.
Our results extend existing theory for single-objective RL and reward-free exploration, and resolve an open question raised by \citet{jin2020reward}.


\section*{Acknowledgement}
This research was supported in part by NSF CAREER grant 1652257, ONR Award N00014-18-1-2364, the Lifelong Learning Machines program from DARPA/MTO, and NSF HDR TRIPODS grant 1934979.

\bibliographystyle{plainnat}
\bibliography{ref}

\newpage
\appendix

\section{Numerical Simulations}\label{append-sec:experiment}

Code for simulations is available at \url{https://github.com/uuujf/MORL}.

\paragraph{Comparison with the Optimal Single-Objective RL Algorithm.}
Figure~\ref{fig:simulation} shows a regret comparison between the proposed \MOVI and the best-in-hindsight policy.
In the experiment we simulate a random multi-objective MDP with $S = 20, A = 5, H = 10, d = 15$, and run the algorithms for $K = 5000$ episodes. 
The best-in-hindsight policy refers to a policy that is fixed across episodes and achieves maximum cumulative rewards in the whole game, i.e., $\arg\max_{\pi} \sum_{k=1}^K V^k_1(x_1;w^k)$.
Note that the best-in-hindsight policy is the optimal policy in the single-objective RL setting;
however it could be much worse than a time-varying policy in the multi-objective RL setting.
Since the information-theoretically adversarial preferences are computationally infeasible to compute, in Figure \ref{fig:simulation} we simply use a randomly generated set of preferences, for which the “best-in-hindsight” policy already performs poorly.
In sum, the plot shows that our algorithm achieves sublinear regret in online MORL, but the best single-objective RL algorithm will incur linear regret in online MORL.

\paragraph{Comparison with Q-Learning.}
Since MORL is connected to task-agnostic exploration \citep{zhang2020task} in the exploration setting (see discussions in Section \ref{sec:explore}), it is tempted to think that the Q-learning method studied in \citet{jin2018q,zhang2020task} could also work in the setting of online MORL.
However this is refuted by Figure \ref{fig:q-learning}.  We simulate a random multi-objective MDP with $S = 20, A = 5, H = 10, d = 15$, and run the algorithms for $K = 5000$ episodes. 
The Q-learning algorithm is specified by Algorithm 1 in \citet{jin2018q}, excepted that the reward is now a linear scalarization of a preference vector and a multi-objective reward vector.
Figure \ref{fig:q-learning} shows that Q-learning cannot achieve a sublinear regret in the setting of online MORL.

\begin{figure}
    \centering
    \includegraphics[width=0.5\linewidth]{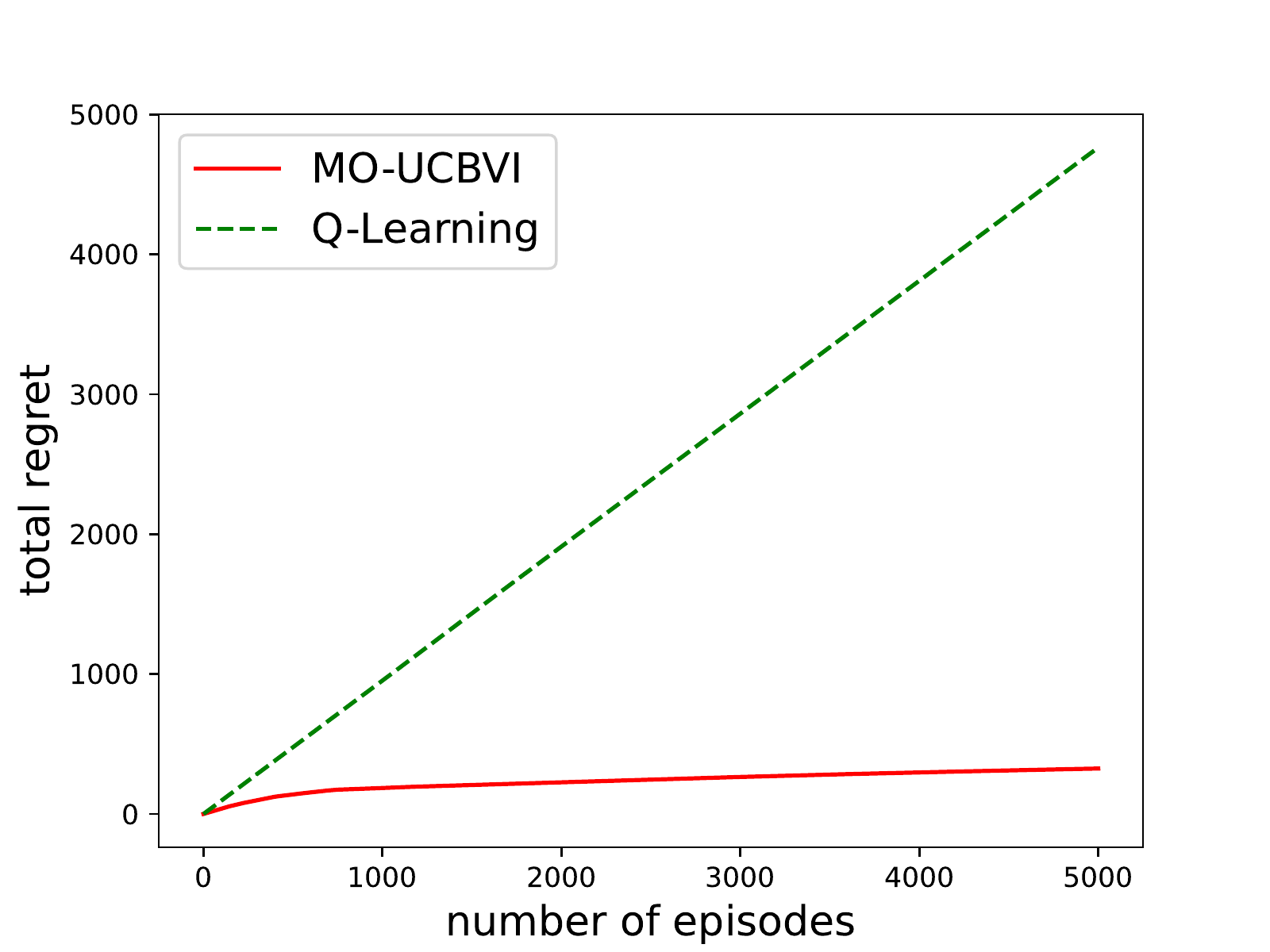}
    \caption{\small \MOVI vs. Q-learning \citep{jin2018q} in a simulated random multi-objective MDP. The plot suggests that \MOVI achieves sublinear regret but Q-learning incurs linear regret.}
    \label{fig:q-learning}
\end{figure}

\paragraph{The Effect of Number of Objectives.}
Figure \ref{fig:number-of-objectives} shows the performance of \MOVI with different number of objectives.
In the experiment we simulate a random multi-objective MDP with $S = 20, A = 5, H = 10$, and number of objectives $d \in \set{1,5,15,20,30}$, and run \MOVI for $K = 5000$ episodes. 
As shown in Figure \ref{fig:number-of-objectives}, the regret will increase as the number of objectives increases; moreover, in all settings \MOVI achieves a sublinear regret.

\begin{figure}
    \centering
    \includegraphics[width=0.5\linewidth]{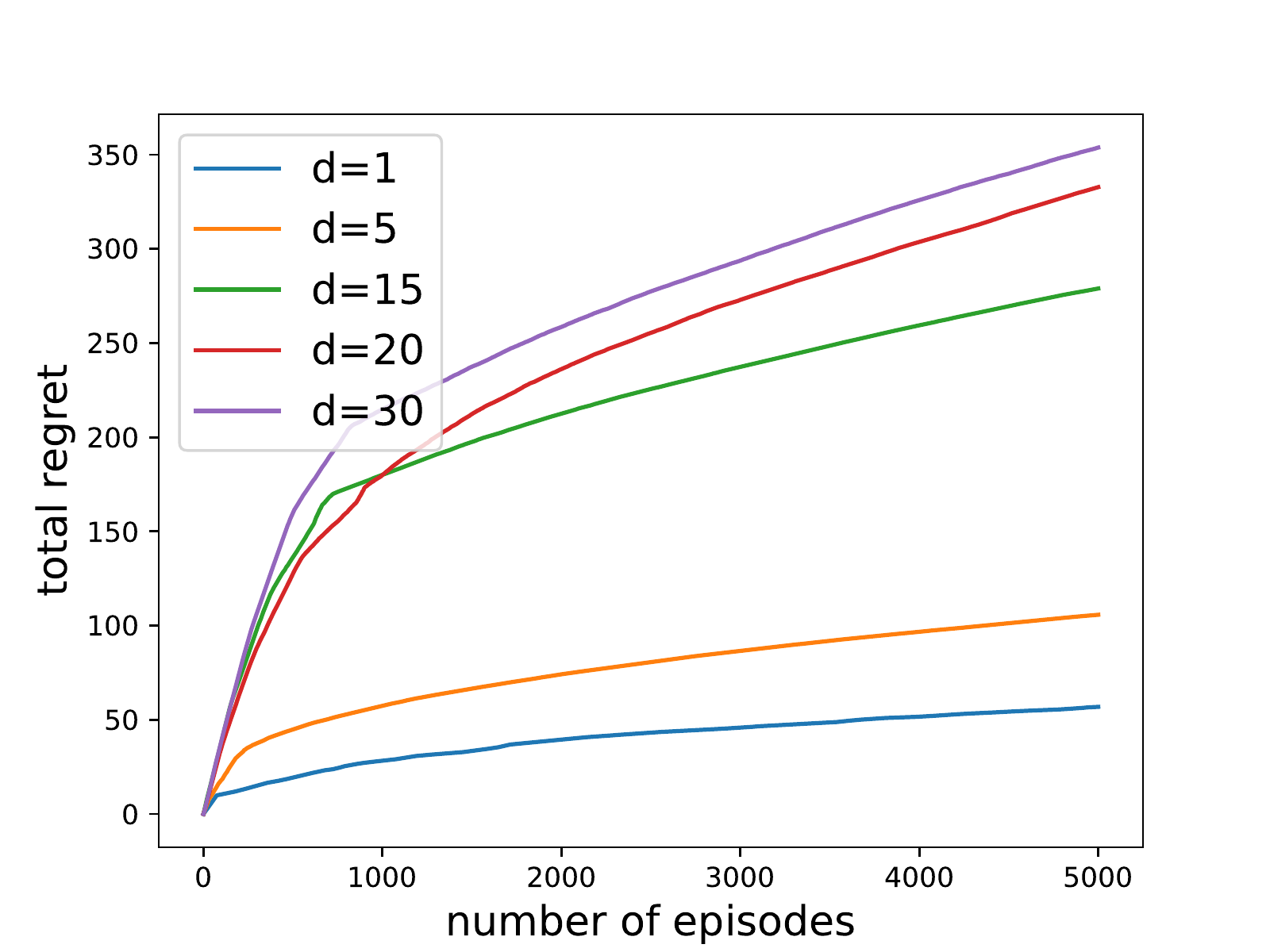}
    \caption{\small The effect of number of objectives. The plot shows the performance of \MOVI in a simulated random multi-objective MDP with different number of objectives. The plot suggests that regret increases as number of objectives increases, but \MOVI achieves sublinear regret in all cases, as predicted by Theorem~\ref{thm:regret-informal}.}
    \label{fig:number-of-objectives}
\end{figure}

\section{Proof of Theorem~\ref{thm:regret-informal}}\label{append-sec:proof-regret}
Our proof utilizes techniques from \citep{zanette2019tighter}.

\paragraph{Notations.}
For two functions $f(x) \ge 0$ and $g(x) \ge 0$ defined for $x \in [0, \infty)$,
we write $f(x) \lesssim g(x)$ if $f(x) \le c\cdot g(x)$ for some absolute constant $c > 0$; 
we write $f(x) \gtrsim g(x)$ if $g(x) \lesssim f(x)$;
and we write $f(x) \eqsim g(x)$ if $f(x) \lesssim g(x)\lesssim f(x)$.
Moreover, we write $f(x) = \Ocal (g(x))$ if $\lim_{x\to\infty} f(x) / g(x) < c$ for some absolute constant $c  >0$;
we write $f(x) = \Omega(g(x))$ if $g(x) = \Ocal (f(x))$;
and we write $f(x) = \Theta(g(x))$ if $f(x) = \Ocal (g(x))$ and $g(x) = \Ocal(f(x))$.
To hide the logarithmic factors, we write $f(x) = \widetilde{\Ocal} (g(x)) $ if $f(x) = \Ocal(g(x) \log^d x) $ for some absolute constant $d > 0$.
For $a,b \in \Rbb$, we write $a\land b := \min\{a,b\}$ and $a\lor b := \max\{a,b\}$.
We will use $\iota$ to denote a general logarithmic factor.
Let $\pi^k$ be the planning policy at the $k$-th episode, i.e., a greedy policy that maximizes $Q^k_h(x,a)$.
Let $w^k_h(x,a) := \Pbb \set{(x_h, a_h) = (x,a) \mid \pi^k, \Pbb}$ and $w^k(x,a) := \sum_{h} w^k_h(x,a) $.

\subsection{Good Events and Useful Lemmas}
Fix $\epsilon$ to be a constant to be determined later (we will set $\epsilon = 1/K$).
Consider the following events: 
\begingroup
\allowdisplaybreaks
\begin{gather}
    G_H := \set{\forall x,a,h,k,w,\  \abs{ \eval{ (\widehat{\Pbb}^k - \Pbb)V^*_{h+1}}_{x,a,w} } \le 2\epsilon + \sqrt{\frac{(d\land S)H^2 }{2 N^k(x,a)} \log \frac{6 H^2SAK}{\delta\epsilon} } }, \label{eq:GH} \tag{$\mathrm{ G_H} $} \\
    \begin{aligned}
        G_{\widehat{B}} &:= \Bigg\{\forall x,a,h,k,w,\  \abs{ \eval{(\widehat{\Pbb}^k - \Pbb)V^*_{h+1}}_{x,a,w} } \le \\
        & \qquad \qquad 2\epsilon + \sqrt{\frac{2(d\land S) \norm{V^*_{h+1}}^2_{\widehat{\Pbb}^k}}{N^k(x,a)} \log\frac{6 H^2 SAK}{\delta\epsilon}}  + \frac{7(d\land S)H}{3N^k(x,a)}\log\frac{6H^2SAK}{\delta\epsilon} \Bigg\},
    \end{aligned} \label{eq:GhatB} \tag{$\mathrm{ G_{\widehat{B}}}$} \\
     G_V := \set{\forall x,a,h,k,w, \ \abs{\eval{ \norm{V^*_h}_{\widehat{\Pbb}^k} - \norm{V^*_h}_{{\Pbb}}  }_{x,a,w} } \le 2\epsilon + \sqrt{\frac{4(d\land S)H^2}{N^k(x,a)}\log\frac{6H^2SAK}{\delta\epsilon}} }, \label{eq:GV} \tag{$\mathrm{G_V}$} \\
    \begin{aligned}
        G_{P} &:= \Bigg\{\forall x,a,y,k,\  \abs{\widehat{\Pbb}^k(y\mid x, a) - \Pbb(y\mid x, a)} \le \\
        &\qquad\qquad\qquad \sqrt{\frac{2\Pbb(y\mid x, a) }{N^k(x,a)}\log\frac{2S^2AK}{\delta}} 
        + \frac{2}{3N^k(x,a)}\log\frac{2S^2AK}{\delta}   \Bigg\}, 
    \end{aligned}\label{eq:GP} \tag{$\mathrm{G_P}$} \\
    \begin{aligned}
        G_{\widehat{P}} &:= \Bigg\{\forall x,a,y,k,\  \abs{\widehat{\Pbb}^k(y\mid x, a) - \Pbb(y\mid x, a)} \le \\
        &\qquad\qquad\qquad \sqrt{\frac{2\widehat{\Pbb}^k(y\mid x, a) }{N^k(x,a)}\log\frac{2S^2AK}{\delta}} 
        + \frac{7}{3N^k(x,a)}\log\frac{2S^2AK}{\delta}   \Bigg\}, 
    \end{aligned}\label{eq:GhatP} \tag{$\mathrm{G_{\widehat{P}}}$} \\
    G_N := \set{\forall x,a,k,\ N^k(x,a) \ge \half \sum_{j < k} w^j(x,a) - H \log \frac{HSA}{\delta} }. \label{eq:GN} \tag{$\mathrm{G_N}$}
\end{gather}
\endgroup

\begin{lem}[Probability of good events.]\label{lem:good-event-prob}
Each of the good event holds with probability at least $1-\delta$.
In particular, they hold simultaneously with probability at least
\[
\Pbb\set{ G_H \cap G_{\widehat{B}} \cap G_{V} \cap G_P \cap G_{\hat{P}} \cap  G_N } \ge 1-6\delta.
\]
\end{lem}
\begin{proof}
We only need to prove that each of the good event holds with probability at least $1-\delta$.

We first show that $\Pbb\set{G_H} \ge 1-\delta$ by Hoeffding's inequality and a covering argument.
In particular by Hoeffding's inequality we have that for fixed $x,a,h,k,w$,
\begin{equation*}
    \abs{ \eval{(\widehat{\Pbb}^k- \Pbb) V^*_{h+1}}_{x,a,w} } \le \sqrt{\frac{H^2}{2N^k(x,a)}\log\frac{2}{\delta}}
\end{equation*}
holds with probability at least $1-\delta$.
We first apply a covering argument for the preferences set.
We consider an $\frac{\epsilon}{H}$-covering $\Ccal$ for the unit ball $\set{w\in\Rbb^d:\norm{w}_1 \le 1}$, then $\abs{\Ccal} \le \bracket{\frac{3H}{\epsilon}}^d$, and for any $w$ in the ball, there exists $w' \in \Ccal$ such that $\norm{w-w'}_1 \le \frac{\epsilon}{H}$.
Then a union bound on $w \in \Ccal$ and $(x,a,k,h) \in \Scal \times \Acal \times [K] \times [H]$ yields that with probability at least $1-\delta$, the following holds for every $(x,a,k,h,w)\in \Scal \times \Acal \times [K] \times [H] \times \Ccal$: 
\begin{equation*}
 \abs{ \eval{(\widehat{\Pbb}^k- \Pbb) V^*_{h+1}}_{x,a,w}} \le \sqrt{\frac{dH^2}{2N^k(x,a)}\log\frac{6 H^2SAK}{\delta\epsilon}}
\end{equation*}
Now consider an arbitrary $w \in \Wcal$ (hence in the unit ball), let $w'\in \Ccal$ be such that $\norm{w-w'}_1 \le \frac{\epsilon}{H}$, then by Lemma~\ref{lem:w-continuity} we have that with probability at least $1-\delta$,
\begin{align*}
    \abs{\eval{(\widehat{\Pbb}^k - \Pbb)V^*_{h+1}}_{x,a,w}} 
    &\le \abs{ \eval{\widehat{\Pbb}^k V^*_{h+1}}_{x,a,w} - \eval{\widehat{\Pbb}^k V^*_{h+1}}_{x,a,w'}} \\
    &\qquad + \abs{\eval{(\widehat{\Pbb}^k-\Pbb) V^*_{h+1}}_{x,a,w'}} 
    + \abs{\eval{{\Pbb} V^*_{h+1}}_{x,a,w'} - \eval{{\Pbb} V^*_{h+1}}_{x,a,w}} \\
    &\le 2\epsilon + \sqrt{\frac{d H^2}{2N^k(x,a)}\log\frac{6H^2 S A K}{\delta\epsilon}}
\end{align*}
holds for every $(x,a,k,h,w) \in \Scal \times \Acal \times [K] \times [H] \times \Wcal$.
Similarly, we can apply the above covering argument for the value function set by considering an $\epsilon$-covering $\Ccal$ for the $H$-ball $\set{\vB \in\Rbb^S: \norm{\vB}_\infty \le H}$, then $\abs{\Ccal} \le \bracket{\frac{3H}{\epsilon}}^S$, then we obtain that with probability at least $1-\delta$,
\begin{align*}
    \abs{\eval{(\widehat{\Pbb}^k - \Pbb)V^*_{h+1}}_{x,a,w}} 
    \le 2\epsilon + \sqrt{\frac{S H^2}{2N^k(x,a)}\log\frac{6H^2 S A K}{\delta\epsilon}}
\end{align*}
holds for every $x,a,k,h$ and every value functions, hence for every preference-induced value functions.
The two inequalities together show that  $\Pbb\set{G_H} \ge 1-\delta$.

Similarly, we can apply the covering arguments with the empirical Bernstein's inequalities \citep{maurer2009empirical} to show that $\Pbb\set{G_{\widehat{B}}} \ge 1-\delta$.

$\Pbb\set{G_{V}} \ge 1-\delta$ is proved by the covering arguments with Theorem 10 from \citep{maurer2009empirical}.

$\Pbb\set{G_{{P}}} \ge 1-\delta$ and $\Pbb\set{G_{\widehat{P}}} \ge 1-\delta$ are proved by Bernstein and empirical Bernstein's inequalities \citep{maurer2009empirical}, respectively.


Finally, $\Pbb\set{G_{N}} \ge 1-\delta$ is due to Lemma F.4 from \citep{dann2017unifying}.


\end{proof}

\begin{lem}[Continuity]\label{lem:w-continuity}
For every $h$, we have
\(
    \abs{V^*_h(x;w) - V^*_h(x;w')} \le (H-h + 1)\norm{w-w'}_1.
\)
\end{lem}
\begin{proof}
We prove it by induction.
For $H+1$, we have $V^*_{H+1}(x;w) = V^*_{H+1}(x;w') = 0$. 
Now suppose
\[
    \abs{V^*_{h+1}(x;w) - V^*_{h+1}(x;w')} \le (H-h)\norm{w-w'}_1,
\]
and consider $h$.
Without loss of generality, suppose $V^*_h(x;w) \ge V^*_h(x;w')$. Then
\begin{align*}
    \abs{V^*_h(x; w) - V^*_h(x; w')} 
    &= V^*_h(x; w) - V^*_h(x; w') \\
    &= \max_a Q^*_h(x, a; w) - \max_{a'}Q^*_h(x,a; w') \\
    &\le Q^*_h(x, a; w) - Q^*_h(x,a; w') \qquad (\text{let $a= \arg\max_a Q^*_h(x, a; w)$}) \\
     &\le \abs{\abracket{ w-w', \rB_h(x,a)}} + \sum_{y\in \Scal}\prob{y \mid x,a} \bracket{ V^*_{h+1}(y; w) - V^*_{h+1}(y; w') } \\
     &\le \norm{w-w'}_1 + (H-h)\norm{w-w'}_1 \\
     &= (H-h+1)\norm{w-w'}_1.
\end{align*}
Hence the conclusion holds.
\end{proof}

\begin{lem}[Bounds for lower order terms]\label{lem:lower-order-terms}
If events \eqref{eq:GP}, \eqref{eq:GhatP} hold, then for every $V_1, V_2$ such that $0 \le V_1(x;w) \le V_2(x;w) \le H$, and for every $x,a,h,k,w$, the following inequalities hold:
\begin{enumerate}
    \item \( \abs{ \eval{ (\widehat{\Pbb}^k - \Pbb) (V_2 - V_1)}_{x,a,w}}
        \le \frac{1}{H} \eval{\Pbb\bracket{V_2 - V_1} }_{x,a,w}  + \frac{2 H^2 S\iota}{N^k(x,a)}; \)
        
    \item \( \abs{ \eval{ (\widehat{\Pbb}^k - \Pbb) (V_2 - V_1)}_{x,a,w}}
        \le \frac{1}{H} \eval{ \widehat{\Pbb} \bracket{V_2 - V_1} }_{x,a,w}  + \frac{3 H^2 S\iota}{N^k(x,a)}; \)
        
    
    \item \(\abs{\eval{(\widehat{\Pbb}^k - \Pbb)(V_2 - V_1 )^2}_{x,a,w}} \le \eval{ \Pbb (V_2 - V_1 )^2}_{x,a,w} + \frac{2 H^2 S \iota}{N^k(x,a)};\)
    
    \item \(\eval{ \norm{ V_2 - V_1  }^2_{\widehat{\Pbb}^k}}_{x,a,w} \le 2 \eval{\Pbb (V_2 - V_1 )^2}_{x,a,w} + \frac{2 H^2 S \iota}{N^k(x,a)};\)
\end{enumerate}
\end{lem}
\begin{proof}
For simplicity in this proof we denote $p(y) := \Pbb(y \mid x,a)$ and $\hat{p}(y) := \widehat{\Pbb}^k(y \mid x,a)$.

For the first inequality,
\begingroup
\allowdisplaybreaks
\begin{align*}
    &\ \abs{ \eval{ (\widehat{\Pbb}^k - \Pbb) (V_2 - V_1) }_{x,a,w} } 
     \le \sum_{y\in \Scal} \abs{\hat{p}^k(y) - p(y)}\bracket{V_2(y;w) - V_1(y;w)} \\
    &\le \sum_{y\in \Scal} \left(\sqrt{\frac{2p(y)\iota}{N^k(x,a)}} +\frac{2\iota}{3N^k(x,a)} \right)\bracket{V_2(y;w) - V_1(y;w)}  \qquad (\text{since \eqref{eq:GP} holds}) \\
    &\le \sum_{y\in \Scal} \left(\frac{p(y)}{H} + \frac{H\iota}{2N^k(x,a)} +\frac{2\iota}{3N^k(x,a)} \right)\bracket{V_2(y;w) - V_1(y;w)} \qquad(\text{use $\sqrt{ab} \le \half (a + b)$}) \\
    &\le \sum_{y\in \Scal} \frac{p(y)}{H}\bracket{V_2(y;w) - V_1(y;w)} +  \frac{2 H^2 S\iota}{N^k(x,a)}  \\
    &= \frac{1}{H}\eval{\Pbb\bracket{V_2 - V_1} }_{x, a,w}  + \frac{ 2  H^2 S\iota}{N^k(x,a)}.
\end{align*}
\endgroup

The second inequality is proved in a same way as the first one, except that in the second step we use event \eqref{eq:GhatP} instead of \eqref{eq:GP}.

For the third inequality,
\begin{align*}
    &\abs{\eval{(\widehat{\Pbb}^k - \Pbb)(V_2 - V_1 )^2}_{x,a,w}} \le 
    \sum_{y} \abs{\hat{p}(y) - p(y)} (V_2(y;w) - V_1(y;w) )^2 \\
    &\le \sum_{y} \bracket{ \sqrt{\frac{2p(y) \iota}{N^k(x,a)} }  + \frac{2\iota}{3N^k(x,a)} } (V_2(y;w) - V_1(y;w) )^2 \qquad (\text{by \eqref{eq:GP}})\\
    &\le \sum_{y} \bracket{ p(y) + \frac{\iota}{2N^k(x,a)}  + \frac{2\iota}{3N^k(x,a)} } (V_2(y;w) - V_1(y;w) )^2 \qquad (\text{use $\sqrt{ab} \le a/2 + b/2$})\\
    &\le \eval{\Pbb (V_2 - V_1)^2}_{x,a,w} + \frac{2H^2S\iota}{N^k(x,a)}.
\end{align*}

For the fourth inequality, notice the following
\[
    \norm{ V_2 - V_1  }^2_{\widehat{\Pbb}^k}
    \le \widehat{\Pbb}^k (V_2 - V_1)^2 \le {\Pbb} (V_2 - V_1)^2 +  (\widehat{\Pbb}^k-\Pbb) (V_2 - V_1)^2,
\]
and use the third inequality.

\end{proof}

\subsection{Proof of the Hoeffding Variant.}
In this part, we follow notations in Algorithm \ref{alg:online} and prove the first claim in Theorem \ref{thm:regret-informal}.

\paragraph{Bonus.}
We set the bonus function in Algorithm \ref{alg:online} to be 
\begin{equation}\label{eq:online-bonus}
    b^k(x, a) := 2\epsilon + \sqrt{\frac{(d\land S)H^2\iota}{2N^k(x, a)}},\qquad \iota := \log\frac{6H^2SAK}{\delta\epsilon}
\end{equation}
where $\epsilon$ will be set as $\epsilon = 1/K$.

\begin{lem}[Optimistic value function estimation]\label{lem:optimistic-value-hoeff}
If event \eqref{eq:GH} holds, then
\(V^*_h (x;w) \le V^k_h (x;w)\).
\end{lem}
\begin{proof}
We prove the lemma by induction over $h \in [H]$.
For $H+1$, 
\(
    V^k_{H+1} (x;w) = V^*_{H+1} (x;w) = 0.
\)
Now suppose $V^k_{h+1} (x;w) \ge V^*_{h+1} (x;w)$ holds for all $(x,k,w) \in \Scal\times [K]\times\Wcal$, and consider $h$.
Let $a = \arg\max_a Q^*_h (x,a;w)$, then by definition we have
\begin{align*}
    V^*_h (x;w) 
    &= Q^*_h (x,a;w) = \abracket{w, \rB_h(x,a)} + \eval{\Pbb V^*_{h+1}}_{x,a,w}, \\
    V^k_h (x;w) 
    & \ge Q^k_h (x,a;w) = H \land \Big( b^k(x,a) + \abracket{w, \rB_h(x,a)} + \eval{\widehat{\Pbb}^k V^k_{h+1}}_{x,a,w} \Big).
\end{align*}
If $Q^k_h (x,a;w) = H$, then $V^k_h (x;w) \ge H \ge V^*_h (x;w)$;
otherwise we have
\begin{align*}
     V^k_h (x;w) - V^*_h (x;w) 
     &\ge b^k(x,a) + \eval{\widehat{\Pbb}^k V^k_{h+1}}_{x,a,w} - \eval{\Pbb V^*_{h+1}}_{x,a;w} \\
     &\ge b^k(x,a) + \eval{\widehat{\Pbb}^k V^*_{h+1}}_{x,a,w} - \eval{\Pbb V^*_{h+1}}_{x,a;w} \qquad (\text{by induction hypothesis})\\
     &\ge 0, \qquad (\text{by \eqref{eq:GH}})
\end{align*}
which completes the induction and finishes the proof.
\end{proof}

\begin{lem}[Per-episode regret decomposition]\label{lem:per-episode-regret-hoeff}
If events \eqref{eq:GH}, \eqref{eq:GP} hold, then for every $k$, it holds:
\begin{equation*}
    V_1^*(x_1;w^k) - V_1^{\pi^k}(x_1;w^k)\lesssim \Ebb_{\Pbb, \pi^k} \sum_{h=1}^H H \land \bracket{\epsilon + \sqrt{\frac{(d\land S)H^2 \iota}{N^k(x,a)} } + \frac{H^2 S\iota}{N^k(x_h,a_h)} }.
\end{equation*}
\end{lem}
\begin{proof}
We first note the following recursion,
\begingroup
\allowdisplaybreaks
\begin{align*}
    &\ V^k_h(x;w) - V^{\pi^k}_h(x;w) 
    = Q^k_h(x,a;w) - Q^{\pi^k}_h(x,a;w) \qquad (\text{set $a = \pi^k_h(x)$}) \\
    &\le  \eval{ b^k +\widehat{\Pbb}^k V^k_{h+1} - \Pbb V^{\pi^k}_{h+1} }_{x,a,w} \\
    &= \eval{ b^k + (\widehat{\Pbb}^k-\Pbb)V^*_{h+1} + (\widehat{\Pbb}^k-\Pbb) (V^k_{h+1}-V^*_{h+1}) + \Pbb(  V^k_{h+1} - V^{\pi^k}_{h+1}) }_{x,a,w} \\
    &\le \eval{ 2 b^k  + (\widehat{\Pbb}^k-\Pbb) (V^k_{h+1}-V^*_{h+1}) + \Pbb(  V^k_{h+1} - V^{\pi^k}_{h+1}) }_{x,a,w} \quad (\text{by \eqref{eq:GH}}) \\
    &\le \eval{ 2 b^k + \frac{2H^2 S\iota}{N^k}  + \bracket{1+\frac{1}{H}} \Pbb(  V^k_{h+1} - V^{\pi^k}_{h+1}) }_{x,a,w} \quad (\text{by Lemma~\ref{lem:lower-order-terms}}),
\end{align*}
\endgroup
and by solving the recursion and noting that $(1+1/H)^H \le e$ we have 
\begin{align*}
    V^k_1(x_1;w^k) - V_1^{\pi^k}(x_1; w^k)
    &\le e \cdot \Ebb_{\Pbb, \pi^k} \sum_{h=1}^H \bracket{2b^k(x_h,a_h) + \frac{2H^2 S\iota}{N^k(x_h,a_h)}} \\
    &\lesssim \Ebb_{\Pbb, \pi^k} \sum_{h=1}^H \bracket{\epsilon + \sqrt{\frac{(d\land S)H^2 \iota}{N^k(x,a)} } + \frac{H^2 S\iota}{N^k(x_h,a_h)} }.
\end{align*}
The claim is proved by using Lemma~\ref{lem:optimistic-value-hoeff} and noting that the value functions are bounded by $H$.
\end{proof}

Now we are ready to prove the first part of Theorem~\ref{thm:regret-informal}.
\begin{thm}[Restatement of Theorem \ref{thm:regret-informal}, Hoeffding part]\label{thm:regret-formal}
    Consider Algorithm~\ref{alg:online} with bonus function \eqref{eq:online-bonus} and $\epsilon = 1/K$, then for any sequence of the incoming preferences $\set{w^1,w^2,\dots,w^K}$,
    the regret~\eqref{eq:regret} of Algorithm~\ref{alg:online} is bounded by
    \[
     \regret (K) \lesssim \sqrt{(d\land S) \cdot H^3 SA K\cdot \log({HSAK}/{\delta})} + H^2 S^2 A \log^2 (HSAK/\delta)
    \]
    with probability at least $1-\delta$.
\end{thm}
\begin{proof}
    First by Lemma~\ref{lem:per-episode-regret-hoeff} we have 
    \begin{align*}
        &\regret(K) = \sum_{k=1}^K V^*_1(x_1; w^k) - V^{\pi^k}_1(x_1; w^k) \\
        &\lesssim \sum_{k=1}^K\sum_{h=1}^H \sum_{x,a} w^k_h(x,a) \cdot H \land \Bigg(  \epsilon + \frac{H^2S\iota}{N^k(x,a)} + \sqrt{\frac{(d\land S)H^2\iota}{N^k(x,a)}}  \Bigg) \\
        &\lesssim 
        \underbrace{ \sum_{k=1}^K\sum_{h=1}^H \sum_{x,a\notin L^k} w^k_h(x,a) \cdot H }_{(i)}
        + \underbrace{ \sum_{k=1}^K\sum_{h=1}^H \sum_{x,a\in L^k} w^k_h(x,a) \cdot \epsilon }_{(ii)} \\
        &\ + \underbrace{ \sum_{k=1}^K\sum_{h=1}^H \sum_{x,a\in L^k} w^k_h(x,a) \cdot \frac{HS\iota}{N^k(x,a)} }_{(iii)} 
        + \underbrace{\sum_{k=1}^K\sum_{h=1}^H \sum_{x,a\in L^k} w^k_h(x,a) \cdot \sqrt{\frac{(d\land S)H^2 \iota}{N^k(x,a)}} }_{(iv)},
    \end{align*}
where we set 
\begin{equation}\label{eq:good-state-actions-hoeff}
    L^k := \big\{(x,a) : \sum_{j < k}w^j(x,a) \ge 2H \iota \big\}.
\end{equation}
We next bound each terms separately.

\emph{Term (i).} By \eqref{eq:good-state-actions-hoeff} we have that
\(
\text{(i)} =  \sum_{k=1}^K \sum_{x,a\notin L^k} w^k(x,a) \cdot H \lesssim H^2 SA \iota.
\)

\emph{Term (ii).} By definition we have that
\(
    \text{(ii)} =  \sum_{k=1}^K \sum_{x,a\in L^k} w^k(x,a) \cdot \epsilon \lesssim HSA  K \epsilon.
\)

\emph{Term (iii).}
The third term is bounded as follows:
\begingroup
\allowdisplaybreaks
\begin{align*}
    \text{(iii)} 
    &= \sum_{k=1}^K \sum_{x,a} w^k(x,a) \cdot  \frac{H^2 S\iota}{N^k(x,a)} \cdot \ind{\sum_{j<k}w^j(x,a) \ge 2H\iota} \qquad(\text{by \eqref{eq:good-state-actions-hoeff}})\\
    &\lesssim \sum_{x,a}\sum_{k=1}^K w^k(x,a) \cdot \frac{H^2 S\iota}{\sum_{j<k}w^j(x,a) - H\iota} \cdot \ind{\sum_{j<k}w^j(x,a) \ge 2H\iota} \qquad(\text{by \eqref{eq:GN}}) \\
    &\lesssim H^2 S\iota\cdot SA \cdot \iota = H^2 S^2 A \iota^2. \qquad(\text{integration trick})
\end{align*}
\endgroup

\emph{Term (iv).}
The fourth term is bounded as follows:
\begin{align*}
    \text{(iv)} 
    &= \sum_{k=1}^K \sum_{x,a} w^k(x,a) \cdot \sqrt{\frac{(d\land S)H^2 \iota}{N^k(x,a)}} \cdot \ind{\sum_{j<k}w^j(x,a) \ge 2H\iota} \qquad(\text{by \eqref{eq:good-state-actions-hoeff}}) \\
    &\lesssim \sum_{x,a} \sum_{k=1}^K  w^k(x,a) \cdot \sqrt{\frac{(d\land S)H^2 \iota}{ \sum_{j<k}w^j(x,a) - H\iota }} \cdot \ind{\sum_{j<k}w^j(x,a) \ge 2H\iota} \qquad(\text{by \eqref{eq:GN}}) \\
    &\lesssim \sqrt{(d\land S)H^2 \iota } \cdot SA \cdot \sqrt{HK} = \sqrt{ (d\land S)H^3 SA K \iota  }.\qquad(\text{integration trick})
\end{align*}

Summing up these terms, setting $\epsilon = 1/K$ and applying Lemma~\ref{lem:good-event-prob}, we complete the proof.

\end{proof}

\begin{algorithm}[H]
    \caption{\MOVI (Bernstein Variant)}
    \label{alg:online-Bernstein}
    \begin{algorithmic}[1]
    \STATE initialize history $\Hcal^0 = \emptyset$, $\iota =  \log (HSAK/\delta)$
    \FOR{episode $k=1,2,\dots,K$}
        \STATE $N^k(x,a),\ \widehat{\Pbb}^k (y \mid x,a)\leftarrow \EmpiTrans(\Hcal^{k-1})$
        \STATE receive a preference $w^k$
         \STATE set $V^k_{H+1}(x ; w^k) = \underline{V}^k_{H+1}(x ; w^k) = 0$ 
        \FOR{step $h = H, H-1, \dots , 1$}
            \FOR{$(x,a) \in \Scal \times \Acal$}
        \STATE compute bonus \( b^k_h(x,a) :\eqsim \sqrt{\frac{(d\land S) \iota}{N^k(x,a)}}\cdot \bracket{ \norm{{V}^k_{h+1}}_{\widehat{\Pbb}^k} + \norm{{V}^k_{h+1} -\underline{V}^k_{h+1} }_{\widehat{\Pbb}^k} }  + \frac{(d\land S)H \iota}{N^k(x,a)}\) 
        \STATE compute bonus \( a^k_h(x,a) :\eqsim \sqrt{\frac{(d\land S) \iota}{N^k(x,a)}}\cdot \bracket{ \norm{\underline{V}^k_{h+1}}_{\widehat{\Pbb}^k} + \norm{{V}^k_{h+1} -\underline{V}^k_{h+1} }_{\widehat{\Pbb}^k} }  + \frac{(d\land S)H \iota}{N^k(x,a)}\) 
                \STATE $Q^k_h(x,a ; w^k) = \min\set{H, \abracket{ w^k, \rB_h(x, a) } + b^k(x,a) + {\widehat{\Pbb}^k_h V^k_{h+1}}\big|_{x,a, w^k} }$
                \STATE $V^k_h(x ; w^k)=\max_{a\in \Acal}Q^k_h(x,a; w^k)$
                \STATE $\pi^k_h(x) = \arg\max_a Q^k_h(x,a ; w^k)$
                \STATE $\underline{Q}^k_h(x,a ; w^k) = \max\set{0, \abracket{ w^k, \rB_h(x, a) } - a^k(x,a) + {\widehat{\Pbb}^k_h \underline{V}^k_{h+1}}\big|_{x,a, w^k} }$
                \STATE $\underline{V}^k_h(x ; w^k)=\underline{Q}^k_h(x,\pi^k_h(x); w^k)$
            \ENDFOR
        \ENDFOR
        \STATE receive initial state $x_1^k = x_1$
        \FOR{step $h=1,2,\dots,H$}
            \STATE take action $a^k_h = \pi^k_h(x^k_h)$, and obtain a new state $x^{k}_{h+1}$
        \ENDFOR 
        \STATE update history $\Hcal^k = \Hcal^{k-1} \cup \{x^k_h, a^k_h\}_{h=1}^H$
    \ENDFOR
    \end{algorithmic}
\end{algorithm}

\subsection{Proof of Bernstein Variant}
In this part, we follow notations in Algorithm \ref{alg:online-Bernstein} and prove the second claim in Theorem \ref{thm:regret-informal}.

For a (estimated) transition kernel $\Pbb$ and a (estimated) value function $V_h(y ; w)$, define its one-step variance as 
\[
\eval{ \norm{V_h}^2_{\Pbb} }_{x,a,w} := \sum_{y \in \Scal} \Pbb(y \mid x,a) \bracket{ V_h(y; w) - \eval{\Pbb V_h}_{x,a,w} }^2.
\]
Moreover, we often omit $\eval{}_{x,a,w}$ and simply write $\norm{V_h}^2_{\Pbb}$ when $x,a,w$ are clear from the context.

\paragraph{Bonus.}
Recall that 
\begin{equation}\label{eq:online-bonus-bern}
    \begin{aligned}
    a^k_h(x,a,w)
    &:= 2\epsilon + \sqrt{\frac{2(d\land S) \iota}{N^k(x,a)}}\cdot \bracket{ \norm{\underline{V}^k_{h}}_{\widehat{\Pbb}^k} + \norm{{V}^k_{h} -\underline{V}^k_{h} }_{\widehat{\Pbb}^k} }  + \frac{7(d\land S)H \iota}{3N^k(x,a)}, \\
    b^k_h(x,a,w)
    &:= 2\epsilon + \sqrt{\frac{2(d\land S) \iota}{N^k(x,a)}}\cdot \bracket{ \norm{{V}^k_{h}}_{\widehat{\Pbb}^k} + \norm{{V}^k_{h} -\underline{V}^k_{h} }_{\widehat{\Pbb}^k} }  + \frac{7(d\land S)H \iota}{3N^k(x,a)},
    \end{aligned}
\end{equation}
where $\epsilon$ is a parameter to be determined (we will set $\epsilon = 1/K$).

\begin{lem}[Optimistic value estimation]\label{lem:optimistic-value-bern}
    Under event \eqref{eq:GhatB}, we have that for every $x,a,k,h,w$, 
    \begin{itemize}
        \item \(\abs{ \eval{(\widehat{\Pbb}^k - \Pbb)V^*_{h+1}}_{x,a,w} } \le b^k_h(x,a,w) \land a^k_h(x,a,w);  \)
        \item \(
            \underline{V}^k_h(x;w) \le V^*_h(x;w) \le {V}^k_h(x;w)
            \)
    \end{itemize}
   \end{lem}
\begin{proof}
    We prove the claims by induction. For $H+1$, $V^*_{H+1} (x,a;w) = \overline{V}^k_{H+1}(x,a;w) = \underline{V}^k_{H+1}(x,a;w) = 0$, hence the hypotheses hold.
    Now suppose the hypotheses hold for $h+1$, i.e.,
\begin{gather}
    \abs{\eval{(\widehat{\Pbb}^k - \Pbb)V^*_{h+1}}_{x,a,w} } \le b^k_h(x,a,w) \land a^k_h(x,a,w), \label{eq:empirical-value-error-h+1} \\
    \underline{V}^k_{h+1}(x;w) \le V^*_{h+1}(x;w) \le \overline{V}^k_{h+1}(x;w) \label{eq:optimistic-h+1}
\end{gather}
    and consider $h$.
    First we have the following two sets of inequalities:
    \begin{align*}
        \overline{V}^k_h(x;w) - V^*_h(x;w)
        &\ge \overline{Q}^k_h(x,a;w) - Q^*_h(x,a;w)\qquad (\text{set $a = \pi_h^*(x)$}) \\
        &\ge b^k_h(x,a,w) + \eval{\widehat{\Pbb}^k \overline{V}^k_{h+1}}_{x,a,w} - \eval{\Pbb V^*_{h+1}}_{x,a,w} \\
        &\ge b^k_h(x,a,w) + \eval{(\widehat{\Pbb}^k -\Pbb) V^*_{h+1}}_{x,a,w} \qquad (\text{by \eqref{eq:optimistic-h+1}}) \\
        &\ge 0, \qquad (\text{by \eqref{eq:empirical-value-error-h+1}}) 
    \end{align*}
    and
\begin{align*}
    V^*_h(x;w) - \underline{V}^k_h(x;w) 
    &\ge  Q^*_h(x,a;w) - \underline{Q}^k_h(x,a;w) \qquad (\text{set $a = \underline{\pi}_h^*(x)$}) \\
    &\ge a^k_h(x,a,w) + \eval{\Pbb V^*_{h+1}}_{x,a,w} - \eval{\widehat{\Pbb}^k \underline{V}^k_{h+1}}_{x,a,w} \\
    &\ge  a^k_h(x,a,w) + \eval{(\Pbb - \widehat{\Pbb}^k) V^*_{h+1}}_{x,a,w} \qquad (\text{by \eqref{eq:optimistic-h+1}}) \\
    &\ge 0. \qquad (\text{by \eqref{eq:empirical-value-error-h+1}}) 
\end{align*}
These justify that 
\begin{equation}
    \underline{V}^k_{h}(x;w) \le V^*_{h}(x;w) \le \overline{V}^k_{h}(x;w). \label{eq:optimistic-h}
\end{equation}
Second, the above inequalities imply that 
\begin{align}
    \norm{V^*_{h}}_{\widehat{\Pbb}^k} 
    & \le \norm{\overline{V}^k_{h}}_{\widehat{\Pbb}^k} + \norm{\overline{V}^k_{h} -V^*_{h} }_{\widehat{\Pbb}^k}
    \le \norm{\overline{V}^k_{h}}_{\widehat{\Pbb}^k} + \norm{\overline{V}^k_{h} -\underline{V}^k_{h} }_{\widehat{\Pbb}^k},\label{eq:V*-bound-upper-bar} \\
     \norm{V^*_{h}}_{\widehat{\Pbb}^k} 
    & \le \norm{\underline{V}^k_{h}}_{\widehat{\Pbb}^k} + \norm{V^*_{h} - \underline{V}^k_{h}}_{\widehat{\Pbb}^k}
    \le \norm{\underline{V}^k_{h}}_{\widehat{\Pbb}^k} + \norm{\overline{V}^k_{h} -\underline{V}^k_{h} }_{\widehat{\Pbb}^k},\label{eq:V*-bound-lower-bar}
\end{align}
therefore,
\begin{align*}
    \abs{\eval{(\widehat{\Pbb}^k - \Pbb)V^*_{h}}_{x,a,w} }
    &\le 2\epsilon + \sqrt{\frac{2(d\land S) \iota}{N^k(x,a)}}\cdot \norm{V^*_{h}}_{\widehat{\Pbb}^k}  + \frac{7(d\land S)H \iota}{3N^k(x,a)} \qquad (\text{by \eqref{eq:GhatB}}) \\
    &\le a^k_{h-1}(x,a;w) \land b^k_{h-1}(x,a;w). \qquad (\text{by \eqref{eq:V*-bound-upper-bar}, \eqref{eq:V*-bound-lower-bar} and \eqref{eq:online-bonus-bern}})
\end{align*}
This completes our induction.
\end{proof}

\begin{lem}[Bonus upper bound]\label{lem:bouns-upper-bound}
If events \eqref{eq:GhatB}, \eqref{eq:GV} hold, then we have that
\begin{align*}
    \eval{ b_h^k \lor a_h^k}_{x, a, w} &\le \eval{ 3\epsilon + \sqrt{\frac{2(d\land S)\iota}{N^k}} \Big( \norm{{V}^*_{h+1}}_{{\Pbb}} + 2\sqrt{ \Pbb({V}^k_{h+1} - \underline{V}^k_{h+1} )^2 }   \Big) + \frac{15 H S \iota}{N^k} }_{x,a,w}.
\end{align*}
\end{lem}
\begin{proof}
    We first prove the bound for $b_h^k(x, a, w)$. Note that 
\begin{align*}
    &\ \norm{{V}^k_{h+1} }_{\widehat{\Pbb}^k} + \norm{{V}^k_{h+1} - \underline{V}^k_{h+1} }_{\widehat{\Pbb}^k} 
    \le \norm{{V}^*_{h+1} }_{\widehat{\Pbb}^k} + \norm{{V}^k_{h+1} - {V}^*_{h+1} }_{\widehat{\Pbb}^k} + \norm{{V}^k_{h+1} - \underline{V}^k_{h+1} }_{\widehat{\Pbb}^k} \\
    &\le  \norm{{V}^*_{h+1} }_{\widehat{\Pbb}^k} +  2\norm{{V}^k_{h+1} - \underline{V}^k_{h+1} }_{\widehat{\Pbb}^k} \qquad(\text{by Lemma~\ref{lem:optimistic-value-bern}}) \\
    &\le 2\epsilon + \norm{{V}^*_{h+1} }_{{\Pbb}} + \sqrt{\frac{4 (d\land S)H^2 \iota }{N^k(x,a)}} +  2\norm{{V}^k_{h+1} - \underline{V}^k_{h+1} }_{\widehat{\Pbb}^k} \qquad(\text{by \eqref{eq:GV}}) \\
    &\le \eval{2\epsilon + \norm{{V}^*_{h+1} }_{{\Pbb}} + \sqrt{\frac{4 (d\land S)H^2 \iota }{N^k}} + 2\sqrt{2  \Pbb ({V}^k_{h+1} - \underline{V}^k_{h+1})^2 + \frac{2H^2 S \iota}{N^k} }  }_{x,a,w}  (\text{by Lemma~\ref{lem:lower-order-terms}}) \\
    &\le \eval{2\epsilon + \norm{{V}^*_{h+1} }_{{\Pbb}} +  2\sqrt{  \Pbb \bracket{\overline{V}^k_{h+1} - \underline{V}^k_{h+1}}^2 } + \sqrt{ \frac{50 H^2 S \iota}{N^k}  } }_{x,a,w}.
\end{align*}
Then the bound for $b_h^k(x, a, w)$ is obtained by substituting the above into \eqref{eq:online-bonus-bern} and use 
\[\sqrt{\frac{2(d\land S)\iota}{N^k(x,a)}} \cdot 2\epsilon \le \epsilon^2 + \frac{2(d \land S) \iota}{N^k(x,a)} \le \epsilon + \frac{2HS\iota}{N^k(x,a)}.\]
Similarly we can obtain the bound for $a_h^k(x, a, w)$.
\end{proof}

\begin{lem}[Per-episode regret decomposition]\label{lem:per-episode-regret}
If events \eqref{eq:GhatB}, \eqref{eq:GV}, \eqref{eq:GP} hold, then
\begin{equation*}
    \begin{aligned}
& V_1^*(x_1;w^k) - V_1^{\pi^k}(x_1;w^k)  \lesssim \\
&\ \Ebb_{\Pbb,\pi^k} \sum_{h=1}^H H \land \Bigg( \epsilon + \eval{ \sqrt{\frac{(d\land S)\iota}{N^k}} \Big( \norm{{V}^*_{h+1}}_{\Pbb} + \sqrt{\Pbb ({V}^k_{h+1} - \underline{V}^k_{h+1})^2 } \Big) + \frac{H^2 S\iota}{N^k} }_{x_h, a_h, w^k}  \Bigg).
    \end{aligned}
\end{equation*}
\end{lem}
\begin{proof}
We first note the following recursion,
\begin{align*}
    &\ V^k_h(x;w) - V^{\pi^k}_h(x;w) 
    = Q^k_h(x,a;w) - Q^{\pi^k}_h(x,a;w) \qquad (\text{set $a = \pi^k_h(x)$}) \\
    &\le  \eval{ b^k_h +\widehat{\Pbb}^k V^k_{h+1} - \Pbb V^{\pi^k}_{h+1} }_{x,a,w} \\
    &= \eval{ b^k_h + (\widehat{\Pbb}^k-\Pbb)V^*_{h+1} + (\widehat{\Pbb}^k-\Pbb) (V^k_{h+1}-V^*_{h+1}) + \Pbb(  V^k_{h+1} - V^{\pi^k}_{h+1}) }_{x,a,w} \\
    &\le \eval{ 2 b^k_h  + (\widehat{\Pbb}^k-\Pbb) (V^k_{h+1}-V^*_{h+1}) + \Pbb(  V^k_{h+1} - V^{\pi^k}_{h+1}) }_{x,a,w} \quad (\text{by Lemma~\ref{lem:optimistic-value-bern}}) \\
    &\le \eval{ 2 b^k_h + \frac{2H^2 S\iota}{N^k}  + \bracket{1+\frac{1}{H}} \Pbb(  V^k_{h+1} - V^{\pi^k}_{h+1}) }_{x,a,w} \quad (\text{by Lemma~\ref{lem:lower-order-terms}}),
\end{align*}
and by solving the recursion and noting that $(1+1/H)^H \le e$ we have 
\begin{align*}
    &\ V^k_1(x_1;w^k) - V_1^{\pi^k}(x_1; w^k)
    \le e \cdot \Ebb_{\Pbb, \pi^k} \sum_{h=1}^H \bracket{2b^k_h(x_h,a_h) + \frac{2H^2 S\iota}{N^k(x_h,a_h)}} \\
    &\lesssim \Ebb_{\Pbb, \pi^k} \sum_{h=1}^H \bracket{ \epsilon + \eval{ \sqrt{\frac{(d\land S)\iota}{N^k}}\cdot \bracket{ \norm{{V}^*_{h+1}}_{\Pbb} + \sqrt{\Pbb ({V}^k_{h+1} - \underline{V}^k_{h+1})^2 } } + \frac{H^2 S\iota}{N^k} }_{x_h, a_h, w^k}  },
\end{align*}
where the last inequality is by Lemma~\ref{lem:bouns-upper-bound}.
The claim is proved by using Lemma~\ref{lem:optimistic-value-bern} and noting that the value functions are bounded by $H$.

\end{proof}

\begin{lem}[Lower order regrets]\label{lem:lower-order-regrets}
    By setting $\epsilon = 1/K$, we have that
\[
    \sum_{k=1}^K \Ebb_{\Pbb, \pi^k}\sbracket{ \sum_{h=1}^H \eval{\Pbb ({V}^k_{h+1} - \underline{V}^k_{h+1})^2 }_{x_h,a_h,w^k} }
\lesssim H^5 S^2 A \iota^2.
\]
\end{lem}
\begin{proof}
    We first note the following recursion,
    \begingroup
\allowdisplaybreaks
    \begin{align*}
        &\ V^k_{h}(x;w) - \underline{V}^k_h(x;w)
        = Q^k_h(x,a;w) - \underline{Q}^k_h(x,a;w)\quad(\text{set $a=\pi^k_h(x)$}) \\
        &\le \eval{ a^k_h + b^k_h + \widehat{\Pbb}^k (V^k_{h+1} - \underline{V}^k_{h+1}) }_{x,a,w} \\
        &= \eval{ a^k_h + b^k_h  + (\widehat{\Pbb}^k - \Pbb) (V^k_{h+1} - \underline{V}^k_{h+1}) + \Pbb (V^k_{h+1} - \underline{V}^k_{h+1}) }_{x,a,w} \\
        &\le \eval{a^k_h + b^k_h + \frac{2H^2S\iota}{N^k} + \bracket{1+\frac{1}{H}}\Pbb (V^k_{h+1} - \underline{V}^k_{h+1})  }_{x,a,w} \qquad (\text{by Lemma~\ref{lem:lower-order-terms}}) \\
        &\le \eval{ 6\epsilon + 6H \sqrt{\frac{2(d\land S)\iota}{N^k}}  + \frac{30HS\iota}{N^k} +  \frac{2H^2S\iota}{N^k}  +  \bracket{1+\frac{1}{H}}\Pbb (V^k_{h+1} - \underline{V}^k_{h+1}) }_{x,a,w} \\ &\qquad (\text{by Lemmas~\ref{lem:optimistic-value-bern} and \ref{lem:bouns-upper-bound}}) \\
        &\le \eval{ 6\epsilon + \sqrt{\frac{100H^2S\iota}{N^k}} + \frac{32 H^2S\iota}{N^k}  +  \bracket{1+\frac{1}{H}}\Pbb (V^k_{h+1} - \underline{V}^k_{h+1}) }_{x,a,w},
    \end{align*}
    \endgroup
    by solving which and noting that $(1+1/H)^H\le e$ and that the value functions are bounded by $H$, we obtain that 
    \begin{align*}
        V^k_{h}(x;w) - \underline{V}^k_h(x;w)
        &\lesssim \Ebb_{\Pbb, \pi^k} \sum_{t \ge h} H \land \bracket{ \epsilon+ \sqrt{\frac{100H^2S\iota}{N^k(x_h,a_h)}} + \frac{H^2 S\iota}{N^k(x_h,a_h)}  },
    \end{align*}
    therefore
    \begin{align*}
        \bracket{V^k_{h}(x;w) - \underline{V}^k_h(x;w)}^2 
        &\lesssim \bracket{ \Ebb_{\Pbb, \pi^k} \sum_{t \ge h} H \land \bracket{  \epsilon+ \sqrt{\frac{H^2S\iota}{N^k(x_h,a_h)}} + \frac{H^2 S\iota}{N^k(x_h,a_h)}  }}^2 \\
        &\lesssim H \cdot \Ebb_{\Pbb, \pi^k} \sum_{t \ge h} \bracket{ H \land \bracket{  \epsilon+ \sqrt{\frac{H^2S\iota}{N^k(x_h,a_h)}} + \frac{H^2 S\iota}{N^k(x_h,a_h)}  } }^2 \\
        &\lesssim H \cdot \Ebb_{\Pbb, \pi^k} \sum_{t \ge h} H^2 \land \bracket{ \epsilon^2 + \frac{H^2 S \iota}{ N^k(x_h, a_h) } + \frac{H^4 S^2\iota^2}{\bracket{N^k(x_h, a_h)}^2 } }.
    \end{align*}
Then
\begingroup
\allowdisplaybreaks
\begin{align*}
    &\  \text{LHS} := \sum_{k=1}^K \Ebb_{\Pbb, \pi^k}{ \sum_{h=1}^H \eval{\Pbb ({V}^k_{h+1} - \underline{V}^k_{h+1})^2 }_{x_h,a_h,w^k} }\\
    &=\sum_{k=1}^K  \sum_{h=1}^H \sum_{x,a} w^k_h(x,a) \eval{\Pbb ({V}^k_{h+1} - \underline{V}^k_{h+1})^2 }_{x,a,w^k}\\
    &\lesssim H \cdot \sum_{k=1}^K  \sum_{h=1}^H \sum_{x,a} w^k_h(x,a)\cdot \Ebb_{\Pbb, \pi^k} \sum_{t \ge h} H^2 \land \bracket{ \epsilon^2 + \frac{H^2 S \iota}{ N^k(x_h, a_h) } + \frac{H^4 S^2\iota^2}{\bracket{N^k(x_h, a_h)}^2 } }  \\
    &\lesssim H^2 \cdot \sum_{k=1}^K  \sum_{h=1}^H \sum_{x,a} w^k_h(x,a) \cdot H^2 \land \bracket{ \epsilon^2 + \frac{H^2 S \iota}{ N^k(x_h, a_h) } + \frac{H^4 S^2\iota^2}{\bracket{N^k(x_h, a_h)}^2 } }  \\
    &\lesssim \sum_{k,h} \sum_{x,a \notin M^k} w^k_h(x,a) H^4
    +  \sum_{k,h} \sum_{x,a \in M^k} w^k_h(x,a) \Big( H^2 \epsilon^2 + \frac{H^4 S \iota}{ N^k(x, a) } + \frac{H^6 S^2\iota^2}{\bracket{N^k(x, a)}^2 } \Big),
\end{align*}
\endgroup
where we set 
\begin{equation*}
    M^k := \big\{(x,a) : \sum_{j < k} w^j(x,a) \ge 2 H S \iota \big\}.
\end{equation*}
Then by \eqref{eq:GN} and the integration tricks (see the proof of Theorem~\ref{thm:regret-formal-bern}), we obtain 
\begin{align*}
    \text{LHS}
    &\lesssim H^4 \cdot SA \cdot HS\iota + H^2 \epsilon^2 \cdot SA \cdot HK + H^4 S \iota \cdot SA \cdot \iota + H^6 S^2 \iota^2 \cdot SA / (HS\iota)    \\
    &\lesssim H^5 S^2 A \iota^2 + H^3 SAK \epsilon^2 \lesssim H^5 S^2 A \iota^2 ,
\end{align*}
where the last inequality is because $\epsilon = 1/K$.
\end{proof}



Now we are ready to prove the second part of Theorem~\ref{thm:regret-informal}.
\begin{thm}[Restatement of Theorem \ref{thm:regret-informal}, Bernstein part]\label{thm:regret-formal-bern}
    Consider Algorithm~\ref{alg:online-Bernstein} with bonus function \eqref{eq:online-bonus-bern} and $\epsilon = 1/K$, then for any sequence of the incoming preferences $\set{w^1,w^2,\dots,w^K}$,
    the regret~\eqref{eq:regret} of Algorithm~\ref{alg:online-Bernstein} is bounded by
    \[\regret(K) \lesssim \sqrt{ (d\land S)H^2 SA K \iota^2} +  H^{2.5} S^2 A \iota^2, \quad \iota := \log({HSAK}/{\delta}),\]
    with probability at least $1-\delta$.
\end{thm}
\begin{proof}
    First by Lemma~\ref{lem:per-episode-regret} we have 
    \begingroup
    \allowdisplaybreaks
    \begin{align*}
        &\ \regret(K) = \sum_{k=1}^K V^*_1(x_1; w^k) - V^{\pi^k}_1(x_1; w^k) \\
        &\lesssim \sum_{k=1}^K\sum_{h=1}^H \sum_{x,a} w^k_h(x,a)  H \land \Big(   \epsilon + \frac{H^2 S\iota}{N^k}  +  \sqrt{\frac{(d\land S)\iota}{N^k}}\Big( \sqrt{\Pbb ({V}^k_{h+1} - \underline{V}^k_{h+1})^2 }  +  {\norm{{V}^*_{h+1}}_{\Pbb} }\Big) \Big) \\
        &\lesssim 
        \underbrace{ \sum_{k=1}^K\sum_{h=1}^H \sum_{x,a\notin L^k} w^k_h(x,a)  H }_{(i)}
        + \underbrace{ \sum_{k=1}^K\sum_{h=1}^H \sum_{x,a\in L^k} w^k_h(x,a) \epsilon }_{(ii)}
        + \underbrace{ \sum_{k=1}^K\sum_{h=1}^H \sum_{x,a\in L^k} w^k_h(x,a)  \frac{H^2S\iota}{N^k(x,a)} }_{(iii)} \\
        & \quad + \underbrace{\sum_{k=1}^K\sum_{h=1}^H \sum_{x,a\in L^k} w^k_h(x,a) \cdot \sqrt{\frac{(d\land S) \iota}{N^k(x,a)}}\cdot \eval{\sqrt{\Pbb ({V}^k_{h+1} - \underline{V}^k_{h+1})^2 }}_{x,a,w^k} }_{(iv)} \\
        &\quad + \underbrace{\sum_{k=1}^K\sum_{h=1}^H \sum_{x,a\in L^k} w^k_h(x,a) \cdot  \sqrt{\frac{(d\land S)\iota}{N^k(x,a)}}\cdot \eval{\norm{{V}^*_{h+1}}_{\Pbb}}_{x,a,w^k} }_{(v)},
    \end{align*}
    \endgroup
where we set 
\begin{equation}\label{eq:good-state-actions}
    L^k := \big\{(x,a) : \sum_{j < k}w^j(x,a) \ge 2H \iota \big\}.
\end{equation}
We next bound each of these terms separately.

\emph{Term (i).} By \eqref{eq:good-state-actions} we have that
\(
\text{(i)} =  \sum_{k=1}^K \sum_{x,a\notin L^k} w^k(x,a)  H \lesssim H^2 SA \iota.
\)

\emph{Term (ii).} By definition we have that
\(
    \text{(ii)} =  \sum_{k=1}^K \sum_{x,a\in L^k} w^k(x,a)  \epsilon \lesssim HSA  K \epsilon.
\)

\emph{Term (iii).}
The third term is bounded as follows:
\begin{align*}
    \text{(iii)} 
    &= \sum_{k=1}^K \sum_{x,a} w^k(x,a) \cdot  \frac{H^2 S\iota}{N^k(x,a)} \cdot \ind{\sum_{j<k}w^j(x,a) \ge 2H\iota} \qquad(\text{by \eqref{eq:good-state-actions}})\\
    &\lesssim \sum_{x,a}\sum_{k=1}^K w^k(x,a) \cdot \frac{H^2 S\iota}{\sum_{j<k}w^j(x,a) - H\iota} \cdot \ind{\sum_{j<k}w^j(x,a) \ge 2H\iota} \qquad(\text{by \eqref{eq:GN}}) \\
    &\lesssim H^2S\iota\cdot SA \cdot \iota = H^2 S^2 A \iota^2. \qquad(\text{integration trick})
\end{align*}

\emph{Term (iv).} 
By Lemma~\ref{lem:lower-order-regrets} and the integration tricks, we bound the fourth term as follows:
\begin{align*}
     \text{(iv)} 
    &\le \sqrt{(d\land S)\iota} \cdot \sqrt{\sum_{k=1}^K\sum_{h=1}^H \sum_{x,a\in L^k} \frac{w^k_h(x,a)}{N^k(x,a)} } \cdot \sqrt{ \sum_{k=1}^K\sum_{h=1}^H \sum_{x,a\in L^k} w^k_h(x,a) \cdot {\Pbb ({V}^k_{h+1} - \underline{V}^k_{h+1})^2 } } \\
    &\lesssim \sqrt{S\iota} \cdot \sqrt{SA \iota} \cdot \sqrt{H^5 S^2 A \iota^2} \lesssim  H^{2.5} S^2 A \iota^2.
\end{align*}

\emph{Term (v).} The fifth term is the leading term. We proceed to bound this term by noting
\begingroup
\allowdisplaybreaks
\begin{align*}
    \text{(v)} 
    &\le \sum_{k=1}^K\sum_{h=1}^H \sum_{x,a\in L^k} w^k_h(x,a) \cdot \bracket{ \sqrt{\frac{(d\land S)\iota}{N^k(x,a)}}\cdot\bracket{ \norm{{V}^{\pi^k}_{h+1}}_{\Pbb} + \norm{{V}^{*}_{h+1} - {V}^{\pi^k}_{h+1}}_{\Pbb}} } \\
    &\le \sqrt{(d\land S)\iota} \cdot \sqrt{\sum_{k=1}^K\sum_{h=1}^H \sum_{x,a\in L^k} \frac{w^k_h(x,a)}{N^k(x,a)} } \cdot \\
    & \Bigg( 
        \sqrt{ \underbrace{ \sum_{k=1}^K\sum_{h=1}^H \sum_{x,a\in L^k} w^k_h(x,a) \cdot   \norm{{V}^{\pi^k}_{h+1} }^2_{\Pbb} }_{\text{(v1)}}  }
        +\sqrt{ \underbrace{ \sum_{k=1}^K\sum_{h=1}^H \sum_{x,a\in L^k} w^k_h(x,a) \cdot   \norm{{V}^{*}_{h+1} - {V}^{\pi^k}_{h+1}}^2_{\Pbb}  }_{\text{(v2)}}  }
    \Bigg),
\end{align*}
\endgroup
where 
\begin{equation*}
    \text{(v1)} = \sum_{k=1}^K \Var_{\pi^k, \Pbb} \sbracket{ \sum_{h=1}^H \abracket{w^k, \rB_h(x_h, a_h)} } \le H^2 K
\end{equation*}
by the law of total variance and that the cumulative reward cannot exceed $H$, and 
\begin{align*}
    \text{(v2)} 
    &\le \sum_{k=1}^K\sum_{h=1}^H \sum_{x,a\in L^k} w^k_h(x,a) \cdot \eval{\Pbb  ({V}^{*}_{h+1} - {V}^{\pi^k}_{h+1})^2}_{x,a,w^k}\\
    &\le H \cdot \sum_{k=1}^K\sum_{h=1}^H \sum_{x,a\in L^k} w^k_h(x,a) \cdot  \eval{\Pbb({V}^{*}_{h+1} - {V}^{\pi^k}_{h+1})}_{x,a,w^k} \\
    &\le H^2 \cdot \sum_{k=1}^K \bracket{ V^*_1(x_1;w^k) - V^{\pi^k}_1(x_1;w^k) }
    = H^2 \cdot \regret(K).
\end{align*}
Hence, 
\begin{align*}
    \text{(v)} 
    &\lesssim \sqrt{(d\land S)\iota} \cdot \sqrt{SA \iota} \cdot \bracket{ \sqrt{H^2 K} + H \cdot \sqrt{\regret(K)} } \\
    &\lesssim \sqrt{ (d\land S)H^2 SA K \iota^2} + \sqrt{ (d\land S) H^2 SA \iota^2} \cdot \sqrt{\regret(K)}.
\end{align*}

Summing up terms (i) to (v) and choosing $\epsilon = 1/K$, we obtain 
\begin{align*}
\regret(K) &\lesssim H^2 SA \iota + H S A K \epsilon + H^2 S^2 A \iota^2 +  H^{2.5} S^2 A \iota^2 \\
&\qquad + \sqrt{ (d\land S)H^2 SA K \iota^2} + \sqrt{ (d\land S) H^2 SA \iota^2} \cdot \sqrt{\regret(K)} \\
&\lesssim  H^{2.5} S^2 A \iota^2 + \sqrt{ (d\land S)H^2 SA K \iota^2} + \sqrt{ (d\land S) H^2 SA \iota^2} \cdot \sqrt{\regret(K)},
\end{align*}
solving which we obtain 
\[
    \regret(K) \lesssim \sqrt{ (d\land S)H^2 SA K \iota^2} + H^{2.5} S^2 A \iota^2.
\]
Applying Lemma~\ref{lem:good-event-prob} completes the proof.
\end{proof}

\section{Proof of Theorem~\ref{thm:exploration-informal} and Corollary \ref{thm:rfe}}\label{append-sec:proof-explore}
In this section, we follow notations in Algorithm \ref{alg:exploration} and Algorithm \ref{alg:planning} and prove Theorem \ref{thm:exploration-informal}.
Let $\overline{V}^k_h(x;w) := \max_a \overline{Q}^k_h(x,a;w)$.
Let $\pi^k$ be the planning policy at the $k$-th episode and $\bar{\pi}^k$ be the exploration policy at the $k$-th episode.

\paragraph{Bonus.}
We set the bonus functions in Algorithm \ref{alg:planning} and Algorithm \ref{alg:exploration} to be 
\begin{equation}\label{eq:plan-explore-bonus}
    b^k(x, a) := 2\epsilon + \sqrt{\frac{(d\land S)H^2\iota}{2N^k(x, a)}},\
    c^k(x,a) := \frac{3H^2 S\iota}{N^k(x,a)} + 2b^k(x,a), \
    \iota := \log\frac{6H^2SAK}{\delta\epsilon},
\end{equation}
where $\epsilon$ is a parameter to be decided later (we will set $\epsilon = 1/K$).


We note Lemma~\ref{lem:optimistic-value-hoeff} applies for the planning phase.
The follows lemma relates planning error with exploration regret.
\begin{lem}[Planning error]\label{lem:planning-error}
If events \eqref{eq:GH}, \eqref{eq:GhatP} hold, then for every $x,a,h,k,w$,
\[
    V^k_h (x, w) - V_h^{\pi^k} (x,w) \le \bracket{1 + \frac{1}{H} }^{H-h+1}  \cdot \overline{V}^k_h(x).
\]
In particular
\( V^k_1 (x_1; w) - V^{\pi^k}_1 (x_1; w) \le e \cdot \overline{V}^k_1 (x) \).
\end{lem}

\begin{proof}
We prove the conclusion by induction.
For $H+1$, the conclusion holds since both the left high side and the right hand side are zero.
Next we assume the conclusion holds for $h+1$, i.e.,
\begin{equation}\label{eq:plan-error-hypothesis}
    V^k_{h+1} (x, w) - V_{h+1}^{\pi^k} (x,w) \le \bracket{1 + \frac{1}{H} }^{H-h}  \cdot \overline{V}^k_{h+1}(x),
\end{equation}
and consider $h$.
    Recall the definitions in Algorithm~\ref{alg:exploration} and Algorithm~\ref{alg:planning}. Set $a=\pi^k_h(x)$, then we have 
    \begin{align}
        V^k_h (x; w) 
        &= Q^k_h(x,a;w) = H\land \bracket{ \abracket{w, \rB_h(x,a)} + \eval{\widehat{\Pbb}^k V^k_{h+1}}_{x,a,w} + b^k(x,a) }, \label{eq:plan-error-vk}\\
        V^{\pi^k}_h (x;w) 
        &= Q^{\pi^k}_h(x,a;w) =  \abracket{w, \rB_h (x,a)} + \eval{{\Pbb} V^{\pi^k}_{h+1}}_{x,a,w},  \label{eq:plan-error-vpi}
    \end{align}
    and
    \begin{equation*}
        \overline{V}^k_h(x) = \max_{a'}\overline{Q}^k_h(x,a') \ge \overline{Q}^k_h(x,a) = H\land \bracket{  \eval{\widehat{\Pbb}^k \overline{V}^{k}_{h+1}}_{x,a} + c^k(x,a) }.
    \end{equation*}
    If $ \overline{Q}^k_h(x,a) = H$, clearly we have 
    \(
        V^k_h(x;w) - V^{\pi^k}_h (x;w)  \le H \le \overline{V}^k_h(x),
    \)
    thus the conclusion holds.
    In the following suppose  
    \begin{equation}\label{eq:plan-error-vbar}
        \overline{V}^k_h(x) = \max_{a'}\overline{Q}^k_h(x,a') \ge \overline{Q}^k_h(x,a) =  \eval{\widehat{\Pbb}^k \overline{V}^{k}_{h+1}}_{x,a} + c^k(x,a) .
    \end{equation}
Then~\eqref{eq:plan-error-vk} and \eqref{eq:plan-error-vpi} yield
\begingroup
\allowdisplaybreaks
    \begin{align*}
        &\ V^k_h (x, w) - V_h^{\pi^k} (x,w)  
        \le  \eval{\widehat{\Pbb}^k V^k_{h+1}}_{x,a,w}  - \eval{{\Pbb} V^{\pi^k}_{h+1}}_{x,a,w}  + b^k(x,a)\\
        &= \eval{\widehat{\Pbb}^k (V^k_{h+1} - V^{\pi^k}_{h+1})  + ( \Pbb - \widehat{\Pbb}^k ) ( V^*_{h+1} - V^{\pi^k}_{h+1} )    + ( \widehat{\Pbb}^k - \Pbb )V^*_{h+1} }_{x,a,w} + b^k(x,w) \\
        &\le \eval{ \widehat{\Pbb}^k (V^k_{h+1} - V^{\pi^k}_{h+1} ) + ( \Pbb - \widehat{\Pbb}^k  ) ( V^*_{h+1}  - V^{\pi^k}_{h+1} ) }_{x,a,w} + 2b^k(x,a). \qquad (\text{by Lemma~\ref{lem:optimistic-value-hoeff}}) \\
    &\le  \bracket{1 + \frac{1}{H} } \eval{ \widehat{\Pbb}^k \bracket{V^k_{h+1} - V^{\pi^k}_{h+1} } }_{x,a,w} + \frac{3 H^2 S\iota}{N^k(x,a)} + 2b^k(x,a) \qquad (\text{by Lemma~\ref{lem:lower-order-terms}}) \\
    &\le \bracket{1 + \frac{1}{H} } \eval{ \widehat{\Pbb}^k \bracket{ \bracket{1 + \frac{1}{H} }^{H-h} \overline{V}^k_{h+1} } }_{x,a} + \frac{3H^2 S\iota}{N^k(x,a)} + 2b^k(x,a) \qquad (\text{by~\eqref{eq:plan-error-hypothesis}})\\
    &= \bracket{1 + \frac{1}{H} }^{H-h+1} \eval{ \widehat{\Pbb}^k{ \overline{V}^k_{h+1} } }_{x,a} + c^k(x,a) \qquad (\text{by~\eqref{eq:plan-explore-bonus}})\\
    &\le \bracket{1 + \frac{1}{H} }^{H-h+1}  \overline{Q}^k_h(x,a) \le \bracket{1 + \frac{1}{H} }^{H-h+1}  \overline{V}^k_h(x) \qquad (\text{by~\eqref{eq:plan-error-vbar}}).
\end{align*}
\endgroup
Thus the conclusion also holds for $h$. By induction we complete the proof. 
\end{proof}

The following lemma is a consequence of Theorem~\ref{thm:regret-formal} with zero reward.
\begin{lem}[Zero reward regret]\label{lem:zero-reward-regret}
In~\eqref{eq:plan-explore-bonus} set 
$\epsilon = 1/K$, then with probability at least $1-\delta$,
the total value (regret) of Algorithm~\ref{alg:exploration} is bounded by
\[ 
\sum_{k=1}^K \overline{V}^k_1(x_1) \lesssim \sqrt{(d\land S) \cdot H^3 SA K \iota} + H^2 S^2 A \iota^2, \quad \iota:= \log (HSAK/\delta).
\]
\end{lem}
\begin{proof}
Use Theorem~\ref{thm:regret-formal} with preference/reward set to be zero.
\end{proof}

We are ready to prove Theorem~\ref{thm:exploration-informal}.
\begin{thm}[Restatement of Theorem~\ref{thm:exploration-informal}]\label{thm:exploration-formal}
Consider Algorithm~\ref{alg:exploration} and Algorithm~\ref{alg:planning} with bonus function \eqref{eq:plan-explore-bonus} and $\epsilon = 1/K$, and suppose that Algorithm~\ref{alg:exploration} is run for 
    \[ K \eqsim \frac{(d\land S)\cdot H^3 SA \iota}{\epsilon^2}+\frac{H^2S^2 A\iota^2 }{\epsilon}, \quad \iota :=\log\frac{HSA}{\delta\epsilon} \]
    episodes, then Algorithm~\ref{alg:planning} outputs an $(\epsilon, \delta)$-PAC policy for preference-free exploration.
\end{thm}
\begin{proof}
    Recall that in the planning phase, we uniformly sample one of the $K$ policies as our final policy. Thus, with probability at least $1-\delta$,
    \begin{align*}
        V^*_1(x_1; w) - V_1^{\pi}(x_1;w) 
        &= \frac{1}{K} \sum_{k=1}^K \bracket{ V^*_1(x_1; w) - V_1^{\pi^k}(x_1;w)  } \\
        &\le \frac{1}{K} \sum_{k=1}^K \bracket{ V_1^k (x_1; w) - V_1^{\pi^k}(x_1;w) } \qquad (\text{by Lemma~\ref{lem:optimistic-value-hoeff}}) \\
        &\le \frac{e}{K} \sum_{k=1}^K { \overline{V}_1^k (x_1) } \qquad (\text{by Lemma~\ref{lem:planning-error}}) \\
        &\lesssim  \sqrt{\frac{(d\land S)\cdot H^3 SA\iota }{K} } + \frac{H^2S^2 A \iota^2}{K}, \qquad (\text{by Lemma~\ref{lem:zero-reward-regret}}) 
    \end{align*}
    where $\iota = \log (HSAK/\delta)$,
    for the right hand side to be bounded by $\epsilon$, it suffices to set 
    \[
    K \eqsim \frac{(d\land S)\cdot H^3 SA }{\epsilon^2} \log\frac{HSA}{\delta\epsilon} + \frac{H^2S^2 A }{\epsilon}\log^2\frac{HSA}{\delta\epsilon}.
    \]
\end{proof}

\subsection{Proof of Corollary \ref{thm:rfe}}
For the first conclusion, we simply set $d=SA$ and apply Theorem \ref{thm:exploration-informal}.

For the second conclusion about non-stationary MDPs, the proof logic follows with small revisions.

First, as the MDP is non-stationary, we replace $\Pbb$ by $\set{\Pbb_h}_{h=1}^H$. Similar in the algorithms, we need to replace $N^k(x,a)$ by $\set{N^k_h(x,a)}_{h=1}^H$, which represents the number of visits to $(x,a)$ at step $h$ upto episode $k$. Then the empirical transition probability will be estimated by 
\(
\widehat{\Pbb}^k_h ( y \mid x,a) = {N^k_h(x,a,y)}/{N^k_h(x,a)},   
\)
for $N^k_h(x,a) \ge 1$.

Second, note that for stationary MDP, we have $\sum_{(x,a)} N^K(x,a) = HK$, but for non-stationary MDP, this needs to be replaced by $\sum_{(x,a)} N^K_h (x,a) = K$ for every $h$.
Then for non-stationary MDP, the integration tricks used in the proof of Theorem \ref{thm:regret-formal} needs to be revised accordingly. Let us take term (iv) as an example, the revised analysis should be:
\begin{align*}
    \text{(iv)} 
    &= \sum_{k=1}^K \sum_{h=1}^H \sum_{x,a} w^k_h(x,a) \cdot \sqrt{\frac{(d\land S)H^2 \iota}{N_h^k(x,a)}} \cdot \ind{\sum_{j<k}w^j_h(x,a) \ge 2\iota} \qquad(\text{by a revised \eqref{eq:good-state-actions-hoeff}}) \\
    &\lesssim \sum_{h=1}^H \sum_{x,a} \sum_{k=1}^K  w^k_h(x,a) \cdot \sqrt{\frac{(d\land S)H^2 \iota}{ \sum_{j<k}w^j_h(x,a) - \iota }} \cdot \ind{\sum_{j<k} w^j(x,a) \ge 2\iota} \ (\text{by a revised \eqref{eq:GN}}) \\
    &\lesssim \sqrt{(d\land S)H^2 \iota } \cdot HSA \cdot \sqrt{K} = \sqrt{ (d\land S)H^4 SA K \iota  }.\qquad(\text{integration trick})
\end{align*}
Therefore the obtained bound has an enlarged $H$ dependence.

Last, repeating the proof of Theorem \ref{thm:exploration-informal} with the revised notations and the above revised inequalities, we obtain the exploration trajectory complexity for non-stationary MDP, where the order of $H$ is enlarged due to the revised Theorem \ref{thm:regret-formal}.

\section{Proof of Theorem~\ref{thm:exploration-lower-bound-informal}}\label{append-sec:pfe-lower-bound}
In this section, we follow notations in Algorithm \ref{alg:exploration} and Algorithm \ref{alg:planning} and prove Theorem \ref{thm:exploration-lower-bound-informal}.
Our construction of hard instance is based on the results from \citet{jin2020reward}.

The following two definitions are migrated from reward-free exploitation (RFE)~\citep{jin2020reward} to preference-free exploitation (PFE) in MORL.

\begin{defi}[Set of MOMDPs]
    Fix $\Scal, \Acal, H, \rB, \Wcal$ as the state sets, action sets, length of horizon, reward vector, preferences set, respectively.
Then a set of transition probabilities $\Pscr$ induces \emph{a set of MOMDPs}, denoted as $(\Scal, \Acal, H, \Pscr, \rB, \Wcal) := \set{(\Scal, \Acal, H, \Pbb, \rB, \Wcal) : \Pbb \in \Pscr}$.
\end{defi}

\begin{defi}[$(\epsilon, p)$-correctness]
We say a PFE algorithm is $(\epsilon, p)$-correct for a set of MOMDPs $(\Scal, \Acal, H, \Pscr, \rB, \Wcal)$, if with probability at least $1-p$, 
\begin{equation*}
    V_1^{\pi_{w}} (x_1; w; \Pbb) \ge V_1^*(x_1 ;w ; \Pbb) - \epsilon, \ \text{ for all } w\in\Wcal \text{ and } \Pbb \in \Pscr,
\end{equation*}
where $\pi_{w}$ is the policy given by the PFE algorithm for preference $w\in\Wcal$.
\end{defi}

\paragraph{Basic hard instance.}
We specify the following set of MOMDPs $(\Scal, \Acal, 2, \Pscr_{\basic}(s,\epsilon), \rB, \Wcal)$ as a basic hard instance:
\begin{itemize}
    \item A state set $\Scal := \set{s}\cup [d]$, where $s$ is the initial state, and $\abs{\Scal} = d+1$.
    \item An action set $\Acal := [A]$. 
    \item A horizon length $H=2$. 
    \item A $d$-dimensional action-independent reward vector, described by a matrix $\rB := \bracket{0, I_{d\times d}} \in \Rbb^{d \times (d+1)}$, i.e., the $i$-th row $\rB^{(i)}(\cdot) \in \Rbb^{d+1}$ sets reward $1$ at state $i$, and reward $0$ elsewhere. 
    \item A preference set $\Wcal := \set{w\in\Rbb^d, \norm{w}_1 = 1}$. For $w\in\Wcal$, ${w^\top \rB} \in \Rbb^{d+1}$ gives a scalarized reward.
    \item A set of transition probabilities
    \begin{align*}
        \Pscr_{\basic}(s,\epsilon) := \set{\Pbb\; :\; \forall a\in[A],\ i\in[d],\ \abs{\Pbb\bracket{i \given s,a}-\frac{1}{d}} \le \frac{\epsilon}{d},
        \ \Pbb\bracket{i \given i, a} = 1}.
    \end{align*}
    For $\Pbb\in\Pscr_{\basic}(s,\epsilon)$, the initial state is always set to be $s$; then the agent takes an action and transits to $[d]$ with a near uniform distribution; finally, $[d]$ are all absorbing states.
\end{itemize}

For the above basic hard instance, \citet{jin2020reward} gives the following lemma to characterizes its learning complexity.
\begin{lem}[Lower bound on the sample complexity of PFE for the basic hard instance]\label{lem:basic-hard-instance}
Fix $\epsilon \le 1, p\le \frac{1}{2}, A\ge 2$, suppose $d\ge \bigO{\log A}$.
There exists a distribution $\Dcal$ over $\Pscr_{\basic}(s,\epsilon)$, such that any $(\epsilon/12, p)$-correct PFE algorithm $\algo$ for $(\Scal, \Acal, 2, \Pscr_{\basic}, \rB, \Wcal)$ must satisfy
\begin{equation*}
    \Ebb_{\Pbb\sim\Dcal}\expect[\Pbb, \algo]{K} \ge \bigOmg{\frac{dA}{\epsilon^2}},
\end{equation*}
where $K$ is the number of trajectories collected by $\algo$ in the exploration phase.
\end{lem}
\begin{proof}
    See \citet{jin2020reward}, Lemma D.2.
    Note that by the construction of the basic hard instance, the preference-weighted reward recovers any reward that is action-independent, thus for the basic hard instance defined in the language of PFE, it equals to the basic hard instance studied by \citet{jin2020reward}, Lemma D.2 for RFE.
\end{proof}

\paragraph{Hard instance in full version.}
Based on the basic hard instance, we next build the hard instance in full version that witnesses the lower bound in Theorem~\ref{thm:exploration-lower-bound-informal}.
Let $n = 2^{\ell_0}$.
The hard instance is a set of MOMDPs $(\Scal, \Acal, H, \Pscr, \rB, \Wcal)$ specified as follows:
\begin{itemize}
    \item A state set
    \[
        \Scal := \set{(s,\ell): s\in [2^\ell],\ \ell = 0,\dots,\ell_0} \bigcup \set{(s,\ell_0 + 1):s\in[d]},
    \]
    where $(0,0)$ is the initial state.
    The states can be viewed as a $\ell_0$-layer binary tree with $d$ leaves attached to the last layer. 
    Clearly, $\abs{\Scal} = 2^{\ell_0+1} - 1 + d = 2n-1+d = \Theta\bracket{n}$.
    \item An action set $\Acal := [A]$. 
    \item A horizon length $H$.
    \item A $2d$-dimensional action-independent reward vector, described by a matrix 
    \[
        \rB := 
        \begin{pmatrix}
        0_{d\times (n-1)} & A_{d\times n} & 0_{d\times d} \\
        0_{d\times (n-1)} & 0_{d\times n} & I_{d\times d}
        \end{pmatrix} \in \Rbb^{2d \times (2n-1+d)},
    \]
    where $A\in\Rbb^{d\times n}$ is a fixed matrix that we will specify later.
    \item A preference set $\Wcal := \Wcal_A \oplus \Wcal_{\basic} \subset \Rbb^{2d}$.
    Here $\Wcal_{\basic} \subset \Rbb^d$ is the preference set for the basic hard instance and gives the set of the last $d$-dimensions of a preference vector.
    And we define $\Wcal_A := \set{A e_i,\ e_i\in\Rbb^n,\ i=1,\dots, n} \subset \Rbb^d$, which gives the set of the first $d$-dimensions of a preference vector.
    Here $e_i$ is the $i$-th standard coordinate vector for $\Rbb^n$.

    Then ${w^\top \rB \in \Rbb^{2n-1+d}}$ gives a scalarized reward.
    The first $2n-1$ dimensions of the scalarized reward ${w^\top \rB}$ are always zero, i.e., the reward for the states in the $0,\dots,\ell_0$-th layers is always zero.
    The last $d$ dimensions of the scalarized reward is sampled from $\set{ A^\top A e_i,\ i=1,\dots, n}$, which specifies the reward for states in the $(\ell_0+1)$-th layer.
    \item A set of transition probabilities
    \begin{align*}
    \Pscr := \big\{
    & \prob{x_{\ell+1} = (s, \ell+1) \given x_\ell = (s,\ell), a = 1} = 1, \ \ell=0,\dots,\ell_0-1,\\
    & \prob{x_{\ell+1} = (2^\ell + s, \ell+1) \given x_\ell = (s,\ell), a > 1} = 1, \ \ell=0,\dots,\ell_0-1;\\
    & \prob{x_{\ell_0 + 1} \given x_{\ell_0} = (s,\ell_0), a} = \Pbb_{\basic}^s, \ \Pbb_{\basic}^s \in \Pscr_{\basic}(s,\epsilon)
    \big\}.
\end{align*}
Notice the transition probabilities up to the $\ell_0$-th layer are fixed. 
And the transition probability from the $\ell_0$-th layer to the $(\ell_0 + 1)$-th layer is specified by $n$ multiples of the transition probabilities that are defined earlier for the basic hard instance.
In sum, a transition probability $\Pbb \in \Pscr$ corresponds to $n$ multiples of transition probabilities $\Pbb_{\basic}^{s} \in \Pscr_{\basic}(s,\epsilon)$, $s\in[n]$.
\end{itemize}


The following lemma is the key ingredient for obtaining lower bound for PFE.
\begin{lem}\label{lem:jl-embed}
Let $e_1, \dots, e_n$ be the standard coordinate vector for $\Rbb^n$. 
Let $e_0$ be the zero vector.
Suppose $d \ge \frac{200}{\epsilon_1^2}\log (n+1)$.
Then there exists a matrix $A \in \Rbb^{d\times n}$ such that 
\begin{align*}
    \norm{A^\top A e_i - e_i}_\infty \le \epsilon_1,\quad i=1,\dots,n.
\end{align*}
\end{lem}
\begin{proof}
For $d \ge \frac{8}{\epsilon^2}\log (n+1)$, by Johnson-Lindenstrauss lemma~\citep{johnson1984extensions}, there exists matrix $A \in \Rbb^{d\times n}$, such that
\[
    (1-\epsilon)\norm{e_i - e_j}_2 \le \norm{A e_i - A e_j}_2 \le (1+\epsilon)\norm{e_i - e_j}_2,\quad \forall i, j \in \set{0,1,\dots,n}.
\]
Thus for $i\ne j$,
\begin{align*}
    \abracket{Ae_i, Ae_j} & = \frac{1}{2}\bracket{\norm{A(e_i - e_0)}^2_2 + \norm{A(e_j-e_0)}_2^2 - \norm{A(e_i-e_j)}^2_2} \\
    &\begin{cases}
    \le \half \bracket{2(1+\epsilon)^2 - 2(1-\epsilon)^2} = 4\epsilon, \\
    \ge \half \bracket{2(1-\epsilon)^2 - 2(1+\epsilon)^2} = -4\epsilon.
    \end{cases}
\end{align*}
And for $i=j$,
\[
    \abracket{Ae_i, Ae_i} = \norm{A(e_i-e_0)}^2_2 
    \begin{cases}
    \le (1+\epsilon)^2 \le 1+3\epsilon, \\
    \ge (1-\epsilon)^2 \ge 1-3\epsilon.
    \end{cases}
\]
In sum
\[
    \delta_{i,j}-4\epsilon \le \abracket{A^\top A e_i,\ e_j} \le \delta_{i,j} + 4\epsilon.
\]
Note that $A^\top A e_i = \sum_{j=1}^n \abracket{A^\top A e_i, e_j}$.
Thus the above inequality implies 
\(
    \norm{A^\top A e_i - e_i}_{\infty} \le 4\epsilon.
\)
A rescale of $\epsilon$ completes the proof.
\end{proof}

We can then extend results from \citet{jin2020reward} to PFE using Lemma~\ref{lem:jl-embed}.

\begin{lem}[Generalized Lemma D.3 of~\cite{jin2020reward}]\label{lem:visit-critic-state}
Fix $\epsilon < \frac{1}{4}$, $\epsilon_0 < \frac{1}{16H}$ and $\epsilon_1 < \frac{1}{16}$.
Consider $\Pbb \in \Pscr$ and $w = ( A e_s,\, v) \in \Wcal$, where $v\in\Rbb^d$ is nearly uniform, i.e., 
\[
   \norm{v - \frac{1}{d}\onebb    }_\infty \le \frac{\epsilon_0}{d}.
\]
Then if policy $\pi$ is $\epsilon$-optimal, i.e.,
\(
    V_1^{\pi}(\Pbb, w) \ge V_1^{*}(\Pbb, w) - \epsilon,
\)
it must visit state $(s,\ell_0)$ with probability at least $\half$, i.e.,
\[
    \Pbb^{\pi} [x_{\ell_0} = (s, \ell_0)] \ge \frac{1}{2}.
\]
\end{lem}
\begin{proof}
From the preference vector we compute the scalarized reward as 
\( r = \bracket{0, A^\top A e_s, v}, \)
i.e., states in the $0, \dots, (\ell_0-1)$-th layers has zero reward, states in the $\ell_0$-th layer has reward as $A^\top A e_s$, which takes value nearly $1$ at state $(s,\ell_0)$ and value nearly $0$ at other states, and states in the last layer has nearly uniform reward $v$.

Let $\bar{v} := \frac{1}{d}\sum_{y=1}^d v[y]$ be the expected reward in the final layer under a uniform visiting distribution.
Then by construction we have
\(
    \abs{v(y) - \bar{v}} \le \epsilon_0,
\)
which implies
\(
    \abs{v(y) - v(y')} \le 2\epsilon_0.
\)
Therefore,
\[
    \abs{\Pbb v(\cdot) - \Pbb' v(\cdot)} \le 2\epsilon_0,
    \text{ for any two transition probability $\Pbb$ and $\Pbb'$.}
\]
Now we can compute the value of a policy $\pi$:
\begin{align*}
    V_1^\pi 
    &= \abracket{\Pbb^\pi \sbracket{ (\cdot,\ell_0) }, A^\top A e_s(\cdot) }
    + (H-\ell_0-1) \abracket{\Pbb^\pi \sbracket{(\cdot,\ell_0+1)}, v(\cdot)} \\ &\qquad (\text{since the last layer is absorbing}) \\
    &\le (1+\epsilon_1) \Pbb^\pi \sbracket{ (s,\ell_0) } + \epsilon_1 \bracket{1-\Pbb^\pi \sbracket{ (s,\ell_0) }} +  (H-\ell_0-1) \abracket{\Pbb^\pi \sbracket{(\cdot,\ell_0+1)}, v(\cdot)}  \\
    &\qquad (\text{by Lemma~\ref{lem:jl-embed}}) \\
    &\le \epsilon_1 + \Pbb^\pi \sbracket{ (s,\ell_0) } + (H-\ell_0-1) \abracket{\Pbb^\pi \sbracket{(\cdot,\ell_0+1)}, v(\cdot)} .
\end{align*}

On the other hand we can lower bound the value of the optimal policy by a policy such that $x_{\ell_0} = (s,\ell_0)$ is taken with probability $1$ (this is doable since the transition probability before the $\ell_0$-th layer is deterministic):
\[
    V^*_1 
    \ge A^\top A e_s(s) + (H-\ell_0-1) \abracket{\Pbb^* \sbracket{(\cdot,\ell_0+1)}, v(\cdot)} 
    \ge 1-\epsilon_1 + (H-\ell_0-1) \abracket{\Pbb^* \sbracket{(\cdot,\ell_0+1)}, v(\cdot)}.
\]
In sum
\begin{align*}
    \frac{1}{4} 
    &\ge \epsilon 
    \ge V^*_1 - V^{\pi}_1 
    \ge 1-\Pbb^\pi \sbracket{ (s,\ell_0) } - 2\epsilon_1 - 2(H-\ell_0-1)\epsilon_0 \\
    &\ge 1 - 2\epsilon_1 - 2H\epsilon_0 - \Pbb^\pi \sbracket{(s,\ell_0) } 
    \ge \frac{3}{4} - \Pbb^\pi \sbracket{(s,\ell_0) }.
\end{align*}

\end{proof}

\begin{lem}[Lemma D.4 of \cite{jin2020reward}]\label{lem:reduce-to-basic-case}
Suppose $H \ge 2(\ell_0+1)$.
Then a PFE algorithm $\algo$ that is $(\epsilon, p)$-correct for the hard instance induces $n$ PFE algorithms $\algo^{s}, s\in[n]$, which are all $(\frac{\epsilon}{4H}, p)$-correct for the basic hard instance. 
\end{lem}
\begin{proof}
Suppose policy $\pi$ satisfies 
\( V^\pi_1(\Pbb, w) \ge V^*_1(\Pbb, w) - \epsilon,\) for all $\Pbb \in \Pscr$ and $w\in\Wcal$.
Then by setting $w = \bracket{0, A^TA e_s, v}$ for $s=1,\dots, n$, we obtain $n$ sub-policies, each of which visits a corresponding state from $(1,\ell_0), \dots, (n,\ell_0)$ with probability at least $\half$ by Lemma~\ref{lem:visit-critic-state} and induces a near-optimal policy for the basic hard instance. 
The last claim is since
\begin{align*}
    \epsilon 
    &\ge V^*_1(\Pbb, w) - V^{\pi}_1 (\Pbb, w) \\
    &\ge \Pbb^{\pi}\sbracket{x_{\ell_0} = (s,\ell_0)} \cdot (H - \ell_0 - 1)\cdot \bracket{ V_1^*(s; \Pbb_{\basic}^s, v) - V_1^{\pi}(s; \Pbb_{\basic}^s, v)} \\
    &  \ge \half (H - \ell_0 - 1) \cdot \bracket{ V_1^*(s; \Pbb^s_{\basic}, v) - V_1^{\pi}(s; \Pbb^s_{\basic}, v)} \qquad (\text{by Lemma~\ref{lem:visit-critic-state}})\\
    & \ge\frac{H}{4}\cdot \bracket{ V_1^*(s; \Pbb_{\basic}^s, v) - V_1^{\pi}(s; \Pbb_{\basic}^s, v)}, \qquad (\text{since $H \ge 2(\ell_0 + 1)$})
\end{align*}
which implies
\(
     V^{\pi}_1(s; \Pbb_{\basic}^s, v) \ge V^*_1(s; \Pbb_{\basic}^s, w) - \frac{\epsilon}{4H}
\)
holds with probability $p$.

\end{proof}

We are now ready to state
Theorem~\ref{thm:exploration-lower-bound-informal} formally and deliver the proof.
\begin{thm}[Restatement of Theorem~\ref{thm:exploration-lower-bound-informal}]\label{thm:exploration-lower-bound-formal}
Fix $\epsilon, p$.
There exists a set of MOMDPs induced by a set of transition probabilities $\Pscr$, and a distribution $\Dcal$ over $\Pscr$, such that if a PFE algorithm ($\algo$) is $(\epsilon,p)$-correct for the set of MOMDPs, then the number of trajectories $K$ that $\algo$ needs to collect in the exploitation phase must satisfy
\begin{equation*}
    \Ebb_{\Pbb\sim \Dcal} \Ebb_{\Pbb, \algo} [K] \ge \bigOmg{\frac{\min\set{d,S} \cdot H^2 S A}{\epsilon^2}}.
\end{equation*}
\end{thm}
\begin{proof}

Let $K^s$ be the number of visits to state $(s,\ell_0)$ by $\algo$.
Then according to Lemma~\ref{lem:reduce-to-basic-case}, 
there are $n$ induced algorithms $\algo^s, s\in[n]$ that is $(\frac{\epsilon}{4H}, p)$-correct for the basic hard instance, and each of them collects $k^s$ number of trajectories.

However, by Lemma~\ref{lem:basic-hard-instance} there exists a distribution $\Dcal^s$ over $\Pscr_{\basic}(s, \frac{\epsilon}{4H})$ such that
\[
    \Ebb_{\Pbb\sim \Dcal^s}\expect[\Pbb, \algo^s]{K^s} \ge \bigOmg{\frac{dAH^2}{\epsilon^2}}.
\]

Notice that during each episode, $\algo$ can and only can visit one of $(s,\ell_0)$ in the hard instance, thus we have 
\(K = \sum_{s\in[n]} K^s.\)
Making a summation we obtain
\[
    \expect{K} \ge \bigOmg{\frac{d \cdot n AH^2}{\epsilon^2}} = \bigOmg{\frac{d S AH^2}{\epsilon^2}}.
\]

On the other hand, clearly in our construction, we can set $d = n$ where Lemma \ref{lem:jl-embed} holds trivially.
Then by the same procedure we obtain 
\[
    \expect{K} \ge \bigOmg{\frac{n\cdot n AH^2}{\epsilon^2}} = \bigOmg{\frac{S^2 AH^2}{\epsilon^2}}.
\]
In sum we have 
\[
    \expect{K} \ge \bigOmg{\frac{\min\set{d,S}\cdot S AH^2}{\epsilon^2}}.
\]

\end{proof}

\section{Proof of Theorem \ref{thm:regret-lower-bound-informal}}\label{append-sec:regret-lower-bound}
We now prove Theorem \ref{thm:regret-lower-bound-informal} based on Theorem \ref{thm:exploration-lower-bound-formal}.

\begin{thm}[Restatement of Theorem~\ref{thm:regret-lower-bound-informal}]\label{thm:regret-lower-bound-formal}
Fix $S, A, H$. Suppose $K > 0$ is sufficiently large. Then for any algorithm that runs for $K$ episodes, there exists a set of MDPs $\Pscr$ and a sequence of  (necessarily adversarially chosen) preferences $\set{w^1, \dots, w^{K} }$, such that 
    \[\Ebb_{\Pbb \sim \Pscr, \algo} [\regret(K)] \ge \frac{1}{c}\cdot \sqrt{\min\set{d, S} H^2 SA K}\]
    for some absolute constant $c > 1$, where the expectation is taken with respect to the randomness of drawing an MDP and an algorithm collecting a dataset from the chosen MDP during the exploration phase.
\end{thm}
\begin{proof}
Let us fix $d, S, A, H$ and a MDP structure as described in the proof of Theorem~\ref{thm:exploration-lower-bound-formal}. Let $K$ to be a large number, and $\epsilon := \sqrt{\min\set{d,S}H^2 SA / K_0}$ to be a small number.

We only consider the randomness of (i) choosing an MDP with transition $\Pbb$ and (ii) an algorithm collecting a dataset $\Hcal^K$ from the chosen MDP during the exploration phase. In order words, we will take expectation over all the randomness during the planning phase, in particular, the considered regret is understood as 
\begin{equation*}
    \regret (K) := \sum_{k=1}^K \Ebb\sbracket{ V^*_1(x_1;w^k) - V^{\pi^k}_1(x_1;w^k) \mid \Pbb, \Hcal^K}.
\end{equation*}
With this in mind, let us consider the following probability measures:
\begin{itemize}
    \item Let $\PB_{0}(\cdot)$  be the probability measure induced by the randomness of drawing a transition kernel $\Pbb$ from the set of transitions $\Pscr$.
    \item Let $\PB(\cdot \mid \Pbb)$ to denote the conditional probability measure induced by the randomness of running an algorithm $\algo$ on a MDP with a fixed transition kernel $\Pbb$, i.e., the probability measure of collecting a dataset.
    \item Let
    $\PB(\cdot)$ be the joint probability measure induced by the randomness of drawing a transition kernel $\Pbb$, and the randomness of collecting a dataset.
\end{itemize}

\emph{Converting online-MORL algorithms into PFE algorithms.}
Let $\algo$ be an MORL algorithm determined by any fixed rules (but could contain random bits), and let $\adv$ be an (arbitrary) adversary that provides preferences for $\algo$ in the online MORL game. Then we inductively define a preference-free exploration algorithm $\advalgo(K, \Pbb)$, which runs for at most $K$ episodes on an MDP with transition kernel $\Pbb$:
\begin{itemize}
    \item At episode $1$, $\adv$ chooses a preference $w^1$ based on the fixed rule of $\algo$; and $\algo$ interacts with $\Pbb$ under the guidance of $w^1$ to collect dataset $\Hcal^1$;
    \item At episode $k \le K$, $\adv$ chooses a preference $w^k$ based on the fixed rule of $\algo$ and the history $\Hcal^{k-1}$; and $\algo$ interacts with $\Pbb$ under the guidance of $w^k$ and collect dataset $\Hcal^k = \Hcal^{k-1} \cup \set{\text{new empirical observations}}$.
    \item At episode $k \le K$, if the agent receives a sequence of planning preferences $w_1, \dots, w_n$, the agent outputs a sequence of policies $\pi_1, \dots, \pi_n$ respectively according to the planning rule that $\algo$ adopts at the current episode (which is based on history $\Hcal^{k-1}$).
    For simplicity, we denote $\advalgo(\Pbb, k)[\cdot]: \set{w} \to \set{\pi}$ be such a planning process, i.e., $\advalgo(\Pbb, k)[w_i] = \pi_i$.
\end{itemize}

We now extend $\advalgo$ into a ``correct'' PFE algorithm that stops at finite time:
\begin{equation*}
        \widetilde{\advalgo} (\Pbb, k)[\cdot] := 
        \begin{cases}
            \advalgo(\Pbb, k)[\cdot], & 1 \le k \le K, \\
            \PFUCB(\Pbb, k-K)[\cdot], & K < k \le 2 K, \\
            \text{output random policy for any preference}, & k > 2K.
        \end{cases}
    \end{equation*}
To justify the correctness of this augmented algorithm, consider two error random variables: 
    \begin{align*}
        \error(\Pbb, k, \widetilde{\advalgo})[w] &:= \Ebb[ V^*_1 (w; \Pbb) - V_1^{\pi} (w; \Pbb) \mid \Pbb, \Hcal^{k-1}], \quad \text{ for } \pi = \widetilde{\advalgo}(\Pbb, k)[w] ,\\
        \error(\Pbb, k, \widetilde{\advalgo}) &:= \max_{w} \error(\Pbb, k, \widetilde{\advalgo})[w],
    \end{align*}
    and a stopping time
    \begin{equation*}
    \mathring{K} := (2K) \land \min\set{k: \error(\Pbb, k, \widetilde{\advalgo}) \le \epsilon }.
\end{equation*}
Clearly $\mathring{K} \le 2K$ is bounded.
    
\emph{The augmented PFE algorithm $\widetilde{\advalgo}(\Pbb, \mathring{K})[\cdot]$ is correct.}
By construction, we have that for any MDP with fixed transition kernel, if $\widetilde{\advalgo}$ is run for $2K$ episodes, its output becomes the same output of $\PFUCB$ that runs for $K$ episodes, which is $(\epsilon, \delta \le 0.9)$-correct for any preference according to Theorem \ref{thm:exploration-formal}. Mathematically, for any $\Pbb$,
\[
 \PB \set{ \error(\Pbb, 2K, \widetilde{\advalgo}) > \epsilon \mid \Pbb } < \delta \le 0.9,
\]
then over the randomness of choosing a transition kernel $\Pbb \sim \PB_0$, we have that
\begin{align*}
    \PB \set{ \error(\Pbb, 2K, \widetilde{\advalgo}) > \epsilon } &= \Ebb_{\Pbb \sim \PB_0} \sbracket{ \PB \set{   \error(\Pbb, 2K, \widetilde{\advalgo}) > \epsilon \mid \Pbb }  } < 0.9.
\end{align*}
By discussing whether $\mathring{K} < 2K$ (where the error is smaller than $\epsilon$ with probability $1$) or $\mathring{K} = 2K$ (where we apply the above inequality), we have that 
\begin{equation*}
    \PB \set{ \error(\Pbb, \mathring{K}, \widetilde{\advalgo}) > \epsilon } < 0.9.
\end{equation*}
This implies that $\widetilde{\advalgo}$ that runs for $\mathring{K}$ episodes is $(\epsilon, 0.9)$-correct for the set of MDP described in the proof of Theorem \ref{thm:exploration-lower-bound-formal}.
Then by Theorem \ref{thm:exploration-lower-bound-formal}, we have that 
\begin{equation*}
    \Ebb [\mathring{K} ] \ge \frac{1}{c} \cdot \frac{ \min\set{d,S} H^2SA }{\epsilon^2} = \frac{1}{c} \cdot K,
\end{equation*}
for some absolute constant $c >1$.
On the other hand, 
\begin{align*}
\frac{K}{c_1} \le 
    \Ebb [\mathring{K}] 
    = \Ebb \sbracket{ \ind{\mathring{K} > \frac{K}{2c}}\cdot \mathring{K} }+  \Ebb \sbracket{ \ind{\mathring{K} \le \frac{K}{2c}}\cdot \mathring{K} } 
    \le 2K \cdot \PB\set{ \mathring{K} > \frac{K}{2c} } + \frac{K}{2c},
\end{align*}
which implies that 
\[
 \PB\set{ \mathring{K} > \frac{K}{2c} } \ge \frac{1}{4 c},
\]
then by the definition of $\mathring{K}$, we have that 
\[
 \PB\set{ \text{for each $1 \le k \le \frac{K}{2c}$},\  \error(\Pbb, k, \widetilde{\advalgo}) > \epsilon } \ge \PB\set{ \mathring{K} > \frac{K}{2c} } \ge \frac{1}{4 c},
\]
which further yields that for each $1 \le k \le \frac{K}{2c_1}$,
\begin{equation}\label{eq:regret-lb-per-step-error}
    \Ebb \sbracket{ \error(\Pbb, k,\widetilde{\advalgo}) } \ge \frac{1}{4c} \cdot \epsilon,
\end{equation}
and 
\begin{equation}\label{eq:regret-lb-error-part}
\Ebb \sbracket{ \sum_{k=1}^{K/{(2c)}} \error(\Pbb,k, \widetilde{\advalgo})} \ge \frac{1}{4c} \cdot \epsilon \cdot \frac{K}{2c} = \frac{1}{8c^2}\cdot \sqrt{ \min\set{d,S} H^2SA K }.
\end{equation}

\emph{Regret lower bound.}
Recall the definition of $\error(\Pbb, k,\widetilde{\advalgo})$ and note that for $k\le K/(2c) \le K$, $\widetilde{\advalgo}[\Pbb, k]$ is the same as $\advalgo(\Pbb, k)$, which plans in the same way as $\algo$ in the $k$-th episode.

We note that \eqref{eq:regret-lb-error-part} holds for any adversary that chooses its preferences inductively based on the online MORL algorithm.
Now we specify the adversary $\adv$ as follows: for $k \le K/(2c)$, $\adv$ specifies a preference $w^{\adv}$ based on the rule of $\algo$, the dataset $\Hcal^{k-1}$, and the transition kernel $\Pbb$, such that
\[ 
(w^{\adv})^k:=\arg\max_w  \Ebb[ V^*_1 (w; \Pbb) - V_1^{\pi} (w; \Pbb) \mid \Pbb, \Hcal^{k-1}], \quad \text{ for } \pi = \widetilde{\advalgo}(\Pbb, k)[w], 
\] 
then
\[
\error(\Pbb, k, \widetilde{\advalgo})[(w^\adv)^k] = \error(\Pbb, k,\widetilde{\advalgo}).
\]
As a consequence, we have that 
\begin{align*}
    \Ebb[\regret(K)] 
    & \ge \Ebb[\regret({K}/{(2c)})] = \Ebb \sbracket{ \sum_{k=1}^{K/(2c)} \error(\Pbb, k, \widetilde{\advalgo})[(w^\adv)^k] } \\
    &= \Ebb \sbracket{ \sum_{k=1}^{K/(2c)} \error(\Pbb, k,\widetilde{\advalgo}) } 
    \ge \frac{1}{8c^2}\cdot \sqrt{ \min\set{d,S} H^2SA K },
\end{align*}
and a rescaling of the constant completes our proof.

\end{proof}

\section[Discussion on (Zhang el al., 2020a)]{Discussion on \citep{zhang2020task}}\label{append-sec:discuss-zhang}
There is a technical error in the proof of Lemma 2 in \citep{zhang2020task}.
In particular, the sequence considered in Page 11, between equations (16) and (17),
\[
\set{ \ind{k_i\le K} \cdot \sbracket{ \sbracket{\widehat{\Pbb}^{k_1} - \Pbb_h} \sbracket{\overline{V}^{\bar{\pi}_k}_{h+1} - V^{\pi_k}_{h+1} }(s,a) + \bracket{r^{k_1}_h - \Ebb[r_h](s^{k_i}_h, a^{k_i}_h)} } }_{i=1}^{\tau},
\]
is not a martingale difference sequence, since $\pi_k$ depends on the randomness upto episode $k$, however $k_1, \dots, k_{\tau}$ are all no larger than $k$.
Thus Azuma-Hoeffding's inequality cannot be applied for this sequence.

To fix this error, one may consider applying a covering argument and union bound over the value functions, but the obtained bound in their Theorem 1 should be revised to $\widetilde\Ocal\bigl({\log(N) H^5 S^2 A}/{\epsilon^2}\bigr)$.

\end{document}